\documentclass{article}

\usepackage[preprint]{neurips_2025}

\usepackage[utf8]{inputenc} 
\usepackage[T1]{fontenc}    
\usepackage{hyperref}       
\usepackage{url}            
\usepackage{booktabs}       
\usepackage{amsfonts}       
\usepackage{nicefrac}       
\usepackage{microtype}      
\usepackage{xcolor}         
\usepackage{soul}           

\hypersetup{colorlinks=true, linkcolor = blue,
     anchorcolor = blue,
     citecolor = blue,
     filecolor = blue,
     urlcolor = blue}
     
\usepackage{graphicx} 
\usepackage{amsmath,amsthm,mathtools,amssymb,dsfont,bm}
\usepackage{algorithm}
\usepackage{subcaption}
\usepackage{algpseudocode}

\DeclareMathOperator*{\argmin}{arg\,min}
\DeclareMathOperator*{\argmax}{arg\,max}

\theoremstyle{plain}
\newtheorem{thm}{Theorem}

\newtheorem{thmprop}{Proposition}

\theoremstyle{definition}

\newtheorem*{thmrem*}{Remark}
\newtheorem*{thmprop*}{Proposition}

\theoremstyle{plain}
\makeatletter
\newtheorem*{rep@theorem}{\rep@title}
\newcommand{\newreptheorem}[2]{%
\newenvironment{rep#1}[1]{%
 \def\rep@title{#2 \ref{##1} (Restated)}%
 \begin{rep@theorem}}%
 {\end{rep@theorem}}}
\makeatother
\theoremstyle{plain}
\newreptheorem{theorem}{Theorem}
\newreptheorem{prop}{Proposition}
\newreptheorem{lemma}{Lemma}
\newreptheorem{cor}{Corollary}
\theoremstyle{definition}
\newreptheorem{def}{Definition}

\def\lpbts{\textbf{lpbTS}}
\def\mts{\textbf{mTS}}
\def\mucb{\textbf{mUCB}}

\def\reg{\mathrm{Reg}}

\def\E{\mathbb{E}}

\def\V{\mathbb{V}}

\def\bmu{\boldsymbol{\mu}}

\def\hmu{\hat{\mu}}
\def\hbmu{\hat{\boldsymbol{\mu}}}

\def\bbR{\mathbb{R}}

\def\cA{\mathcal{A}}

\def\cD{\mathcal{D}}

\def\cH{\mathcal{H}}

\def\cL{\mathcal{L}}
\def\cM{\mathcal{M}}
\def\cN{\mathcal{N}}

\def\cP{\mathcal{P}}

\def\cZ{\mathcal{Z}}

\newcommand{\KL}{\mathbb{KL}}
\newcommand{\lpbTS}{\textbf{lpbTS}}
\DeclareMathOperator{\argsort}{argsort}

\newtheorem{observation}{Observation}

\title{Latent Preference Bandits}
\author{%
  Newton Mwai \\
  Department of Computer Science and Engineering\\
Chalmers University of Technology and University of Gothenburg\\
SE-41296 Gothenburg, Sweden\\
  \texttt{mwai@chalmers.se} \\
  \And
  Emil Carlsson \\
  Sleep Cycle \\
  Gothenburg, Sweden \\
  \texttt{emil@sleepcycle.com} \\
  \And
  Fredrik D. Johansson \\
  Department of Computer Science and Engineering\\
Chalmers University of Technology and University of Gothenburg\\
SE-41296 Gothenburg, Sweden\\
  \texttt{frejohk@chalmers.se} \\
}

\date{}

\begin{document}

\maketitle

\begin{abstract}
    Bandit algorithms are guaranteed to solve diverse sequential decision-making problems, provided that a sufficient exploration budget is available. However, learning from scratch is often too costly for personalization tasks where a single individual faces only a small number of decision points. Latent bandits offer substantially reduced exploration times for such problems, given that the joint distribution of a latent state and the rewards of actions is known and accurate. In practice, finding such a model is non-trivial, and there may not exist a small number of latent states that explain the responses of all individuals. For example, patients with similar latent conditions may have the same preference in treatments but rate their symptoms on different scales. With this in mind, we propose relaxing the assumptions of latent bandits to require only a model of the \emph{preference ordering} of actions in each latent state. This allows problem instances with the same latent state to vary in their reward distributions, as long as their preference orderings are equal. We give a posterior-sampling algorithm for this problem and demonstrate that its empirical performance is competitive with latent bandits that have full knowledge of the reward distribution when this is well-specified, and outperforms them when reward scales differ between instances with the same latent state. 
\end{abstract}

%
%
\section{Introduction}
\label{sec:introduction}
Personalized decision-making has promised to revolutionize healthcare for decades but its full impact has yet to come. In the abstract, the problem is equivalent to recommendations for media consumption or ad placement, in which algorithms make decisions and observe outcomes (reward) tailored for an individual (problem instance). Bandit and reinforcement learning algorithms are well-suited for such problems and have seen academic and commercial success~\citep{li2010contextual, chapelle2011empirical,Bouneffouf2020, yancey2020sleeping, o2022should}. However, the scale of data varies substantially with applications: a consumer may be exposed to dozens of recommendations in a single session, and hundreds over their subscription to a service, but a patient may try only a handful of treatments. Given their sample-hungry nature, learning a personalized treatment regime using a classical bandit algorithm to explore solutions for a single individual is usually infeasible. 

A promising solution to personalization with short exploration times is to leverage structural similarities between problem instances (e.g., patients). Contextual bandits is a well-studied solution that models the expected reward of actions as fixed functions of an observed context variable~\citep{lattimore2020bandit}. However, this assumes that two instances would yield the same reward on average for the same action, were they in the same situation. In practice, observed contexts are rarely rich enough to capture all individual preferences. For example, in problems where rewards represent the subjective rating of an experience, different people tend to have different internal rating scales, unobserved by the learning algorithm. Latent bandits~\citep{maillard2014latent} attempt to overcome this limitation by allowing the reward function to depend on a latent state, observed partially and noisily through the context and the outcome of actions. If the latent state is small relative to the number of actions, or the reward is a simpler function of the latent state than the context, this structure may be sufficient to substantially reduce exploration times~\citep{hong2020latent,mwai2023fast}. However, existing works on latent bandits assume that the full posterior distribution of the latent state is known at inference time but give little or no guidance for how it can be learned.

\paragraph{Contributions.} We propose \emph{Latent Preference Bandits (LPB)}---a variant of latent bandits where each discrete latent state defines a preference ordering over actions, but not a full distribution of rewards. We show that knowing only the set of possible preference orderings can still substantially lower the difficulty of the problem, as determined by the number of constraints added to an instance-specific lower bound on the asymptotic regret. Moreover, because states only determine preference orderings, we show that LPB accommodates generalization between problem instances (e.g., patients) with different absolute reward scales, but with shared relative preferences. We also show that knowledge of the set of possible preferences gives a bound on the posterior probability of the latent state and leverage this result in a regret minimization algorithm \lpbTS{}, based on sampling from the approximate posterior. We show empirically that it is comparable to latent bandits with fully known reward distributions when instance rewards lie in the same scale, and outperform them when instances differ in absolute reward scales. Our experiments confirm that the benefit of utilizing preference structure over non-latent baselines increases as the number of arms (preference constraints) grows much larger than the number of states.

%
%
\section{Related work}
\label{sec:related}

The latent bandit model was first studied by \citet{maillard2014latent} in the regret minimization setting and has later been revisited for regret minimization by~\citet{atan2018global, hong2020latent, pmlr-v206-pal23a, balciouglu2024identifiable} and for best-arm identification by~\citet{mwai2023fast}. Recent works have also extended the latent bandit to a non-stationary version where the latent variable evolve over time~\citep{hong2020non, nelson2022linearizing, russoswitching}. A common theme in these works is that the reward distributions are determined completely by the latent state. This differs from our setting in the sense that we assume the latent state only defines an \emph{ordering} of the arms.

Bandit problems with preference feedback have been widely studied in dueling bandits~\citep{yue2009interactively, sui2018advancements, bengs2021preference, herman2024} which are a bandit class where action pairs are selected at each round, and the rewards are independent, stochastic preference feedback of which arm is preferred~\citep{sui2018advancements}. The goal is to identify the best arm, or minimize regret via pairwise comparisons~\citep{ailon2014reducing, bengs2021preference}. In dueling bandits, preferences are typically directly observed~\citep{bengs2021preference}, and often modeled with either utility functions~\citep{yue2009interactively} or Probabilistic models like Bradley-Terry Models (BTMs)~\citep{bradley1952rank, vigneau1999analysis}. Our setting differs from dueling bandits in that we observe feedback as absolute numerical rewards for a single action, and preferences are inferred from these, in contrast to being observed directly through relative preference feedback. 

Contextual bandits~\citep{chu2011contextual, agrawal2013thompson, zhou2015survey, lattimore2020bandit} exploit structure between context, actions, and rewards, promising high personalisation. However, they often assume a fixed expected reward across instances of the same context-action pair, which is inapplicable whenever no such context variable exists, even when the action preference is maintained. In our current work, we do not consider rewards structured based on context variables, but focus on the utility of latent preferences instead.

%
%
\section{Multi-armed and latent bandits}
\label{sec:setup}
We study sequential decision-making problems where the goal is to select actions $a \in \cA = \{1, ..., k\}$ to maximize the corresponding reward $R_a$, \emph{regret minimization} problems. The goal is defined with respect to the unknown expectations of rewards $\mu_{a} \coloneqq \E[R_a]$, with the optimal action $a^* = \argmax_{a} \mu_a$ and optimal reward $\mu^* = \mu_{a^*}$. We aim to select actions $a_t$ according to a policy $\pi$ on times steps $t=1, ..., T$ until a horizon $T$ to accumulate as little regret $\reg(T)$ as possible, 
\begin{equation}\label{eq:regret_min}
\underset{\pi}{\text{minimize}} \;\; \reg(T) \quad \mbox{with} \quad  \reg(T) \coloneqq \sum_{t=1}^T\E_{\pi}[\mu^* - R_{A_t}]~.
\end{equation}

A central challenge in multi-armed bandits (MAB) is that solving \eqref{eq:regret_min} requires excessively many trials, especially when the number of arms, $k$ is large. To remedy this, \emph{latent bandits}~\citep{maillard2014latent}, categorize instances based on a discrete \emph{latent} variable $Z \in \{1, ..., m\}$, representing, for example, the disease subtypes of patients with a given condition. In this way, one instance can inform another one with the same latent state. If the conditional distribution of rewards $p(R_a \mid Z=z)$ and marginal distribution $p(Z=z)$ for each latent state $z$ are known or can be learned, these are sufficient to \emph{infer} $Z$ for a new instance and to minimize cumulative regret~\citep{hong2020latent}.

\citet{hong2020latent} proposed the  latent bandit algorithm \mts{} for regret minimization based on Thompson sampling~\citep{thompson1933likelihood,agrawal2012analysis}. At each round $t$, the algorithm samples a state $z_t$ from the posterior distribution (ignoring context variables here)
$$
p(Z=z \mid a_1, ..., a_{t-1}, r_1, ..., r_{t-1}) \propto \prod_{s=1}^{t-1} p(R_{a_s} = r_s \mid z)p(Z=z)
$$
and selects the highest-reward action of that state, $a^*_z = \argmax_a \mu_{a,z}$ where $\mu_{a,z} \coloneqq \E[R_a \mid Z=z]$.

Estimating a full latent-variable model $\mathcal{M} = (p(Z), \{p(R_a \mid Z)\}_{a=1}^k)$ is nontrivial. It may not be identifiable from observable information and may require a very large data set even if it is. 
Moreover, requiring that all instances with state $Z=z$ follow the same reward distribution $p(R_a \mid Z=z)$ prevents instances from having individual \emph{reward scales}, as in the example of subjective ratings from the Introduction. Fortunately, as we will see, knowing the complete likelihood is not necessary to achieve benefits over learning tabula rasa for a new bandit instance. In fact, knowing only the optimal arm in each latent state leads to improved regret bounds when $m<k$. However, imposing additional structure on the rewards of different actions can help distinguish the true latent state from alternatives. 
To this end, we study latent bandits with states defined by \emph{preference orderings} of actions. 

%
%
\section{Latent preference bandits}
\label{sec:method}
We introduce the \emph{latent preference bandit} (LPB) problem and give an instance-dependent lower bound for the best achievable regret which depends on the structured preferences in rewards. Next, we propose a posterior-sampling algorithm to solve the problem and discuss its performance. 

\begin{figure}[t]
    \centering
    \begin{subfigure}{0.30\linewidth}
        \includegraphics[width=\linewidth]{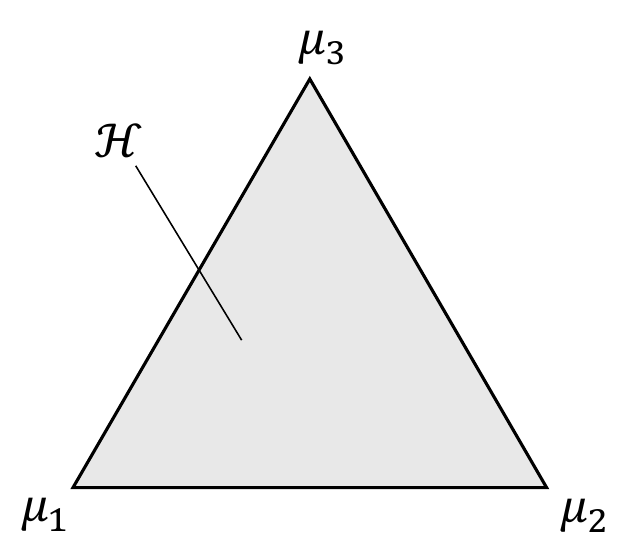}
        \caption{Multi-armed bandit. No constraints on reward means\\\;}    
    \end{subfigure}
    \hfill
    \begin{subfigure}{0.30\linewidth}
        \includegraphics[width=\linewidth]{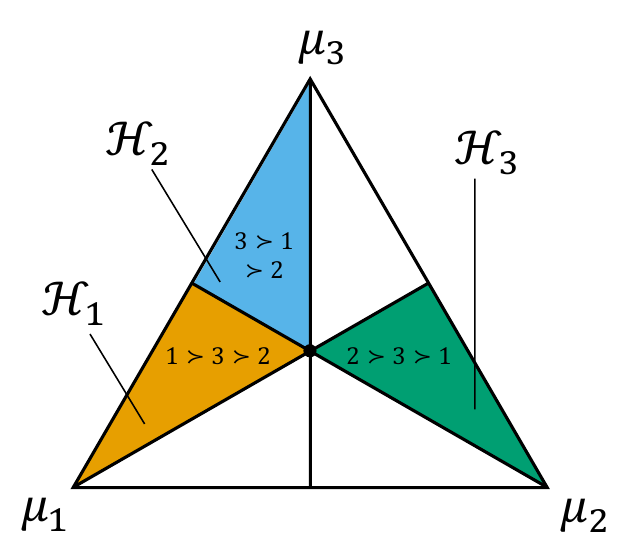}
        \caption{Latent preference bandit (this work). Known possible orderings, unknown reward means}    
    \end{subfigure}
    \hfill
    \begin{subfigure}{0.30\linewidth}
        \includegraphics[width=\linewidth]{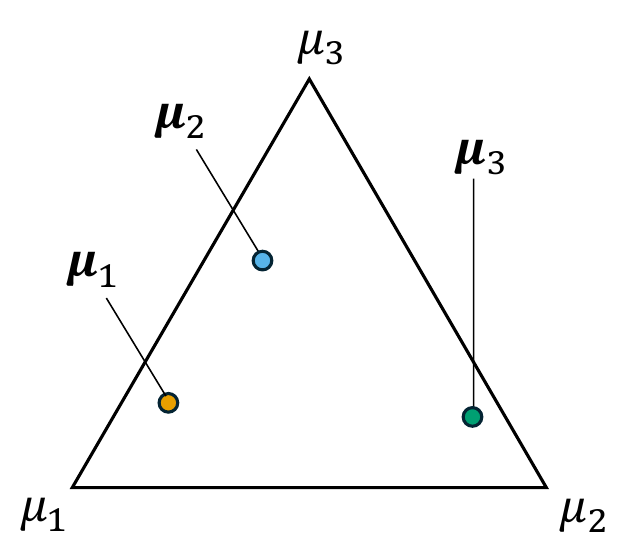}
        \caption{Latent bandit: Fully known possible reward means\\\;}    
    \end{subfigure}    
    \caption{Illustration of the latent preference bandit and related problems for reward means on the 2-simplex $\bmu \in \Delta^{k-1}$. In the MAB problem, no structure is known. In latent bandits, the full vector of reward means $\mu_z$ is known for each latent state $z$. In latent preference bandits, only the set of possible orderings is known (shown as colored segments) but two problem instances with the same latent state $z$ may differ in their means as long as the orderings of their reward means are equal. \label{fig:illustration}}
\end{figure}

In latent preference bandits, a problem instance $(z, \cP)$ is defined by the latent state $z$ and reward distributions for the $k$ actions, $\cP = (P_1, ..., P_k)$ structured according to $z$. For clarity, we focus on Gaussian rewards with equal variance $\sigma^2$ where $P_a = \cN(\mu_a, \sigma^2)$ and, for the remainder of the paper, we represent instances $(z, \bmu)$ by their latent states and reward means $\bmu = [\mu_1, ..., \mu_k]^\top$. Each latent state $z \in [m]$ is associated with a preference ordering of actions $O_z = (o_{z,1}, ..., o_{z,k})$ such that the expected rewards for problem instances $(z, \bmu)$ are ordered according to $O_z$, i.e., $\mu_{o_{z,1}} \geq \mu_{o_{z,2}} \geq \cdots \geq \mu_{o_k}$.\footnote{We drop the $z$-subscript in $\mu_{o_{z,i}}$ for readability when clear from context.} For convenience, let $i_{a,z}$ denote the rank of action $a$ under $O_z$, that is $i_{a,z}=j$ such that $o_{z,j} = a$, and let $a \succeq_z a' \Leftrightarrow i_{a,z} < i_{a',z}$ denote that $a$ is preferred to $a'$ under $z$.

Two problem instances $(z, \bmu), (z', \bmu')$ with the same latent state $z = z'$ are guaranteed to have the same preference orderings but may not have the same distributions of rewards. This allows for modeling individual rating scales, as in the following example: Two patients with a chronic condition share the same subtype of disease $z$ which determines what therapies $a \in \cA$ are preferred over which other therapies. However, the two patients have different tolerance for pain and give different ratings $R_a$ for their symptoms under treatment $a$ even if their relative preferences are the same. 
We let $\cH_z = \{\bmu \in \bbR^k : \forall a \succeq_z a' : \mu_{a} \geq \mu_{a'}\}$ denote the set of all reward parameters consistent with $z$. 

At the start of learning, algorithms are assumed to have access to a model of $\{(z, O_z)\}_{z \in \cZ}$ of preference orderings but have no knowledge of rewards---they only know which preference orderings are possible. We assume that all latent states have unique orderings. 
We illustrate the LPB problem in Figure~\ref{fig:illustration} for the special case of reward means in the 2-dimensional simplex, and compare it to the standard multi-armed bandit (MAB) and the latent bandit problem from \citet{hong2020latent}.
The LPB problem has more structure than MAB but less than full-model latent bandits. To see this, we can adapt an instance-specific lower bound on the asymptotic regret from~\citet{maillard2014latent}, in turn adapted from~\citet{agrawal1989asymptotically}.
\begin{thm}[\citet{agrawal1989asymptotically}]\label{thm:lower_bound}
    Let $(z, \bmu)$ be the true latent state and reward means of an instance and let $\cA_{-} = \cA \setminus \{a^*_{z}\}$ be the set of sub-optimal arms $a$ with optimality gaps $\Delta_a=\mu^* - \mu_a$. Then, for any uniformly good control scheme, i.e., that achieves $\reg(T) = o(T^b)$ for any $b>0$, 
    \begin{equation}\label{eq:lower_bound}
    \liminf_{T\rightarrow \infty} \frac{\reg(T)}{\log T} \geq \min_{w \in \cP(\cA_{-})} \max_{\lambda \in \mathrm{Alt}(z,\bmu)} \frac{\sum_{a\in \cA_-}w_a \Delta_a}{\sum_{a\in \cA_-}w_a \KL(\mu_a || \lambda_a)}, 
    \end{equation}
    where $\mathrm{Alt}(z,\bmu) = \cup_{z : a^*_{z'} \neq a^*_z} \cH_z$,  $\cP(\cA_-)$ is the simplex over $\cA_-$, and $\KL(\mu_a || \lambda_a)$ is the KL-divergence between Gaussian distributions with unit variance and means $\mu_a, \lambda_a$.
\end{thm}
The set $\mathrm{Alt}(z,\bmu)$ of alternative instances contain only parameters $\mu'$ consistent with orderings that are different from that of the true state, $O_z$. If every possible ordering is represented by a latent state ($m=k!$), the problem is difficult, comparable to an MAB without structure. On the other hand, if $m \ll k!$ and the orderings of latent states are random, the closest confusing instance $\lambda^*$, maximizing \eqref{eq:lower_bound}, is likely to have an ordering with a large number of inversions to $\bmu$, the $\KL$ term will be large, and the bound small---the problem is easier to solve. We return to this discussion in Section~\ref{sec:scaling}.
 
%
%
\subsection{Regret minimization with absolute feedback}
\label{sec:regret_min}
We focus on learning from absolute feedback\footnote{In the Appendix, we briefly discuss latent preference bandits with relative (dueling) feedback.} where decision-makers select a single action $a_t \in [k]$ at each round $t$ and observe a stochastic reward $R_t \in \bbR$, generated according to the unknown instance $(z, \bmu)$ with $\bmu \in \cH_z$. Disregarding the latent state, the setting coincides with the classical MAB problem. Thus, the primary target for algorithms that exploit knowledge of the set of possible latent states $z \in [m]$ and their preference orderings $O_z$ is to minimize regret or identify the optimal arm more rapidly than MAB algorithms and other algorithms that exploit the latent structure.

The worst-case asymptotic regret for the MAB problem is well-known to be $\Theta(\sqrt{kT})$ as $T\rightarrow \infty$, achieved by several algorithms, including upper-confidence bound (UCB) maximization and posterior (Thompson) sampling~\citep{lattimore2020bandit}. \citet{hong2020latent} proposed the \mts{} and \mucb{} algorithms for the latent bandit problem with $m$ states and showed that their worst-case asymptotic regret with full knowledge of the posterior of observations, including the reward distribution, is $O(\sqrt{mT\log T})$. This matches the MAB result up to log factors, but with $m$ instead of $k$, which can be very beneficial if $m < k$. However, with knowledge of the set of possible latent states, this can be achieved simply by restricting the action set.

\begin{thmprop}\label{prop:regret_upper}Consider the following algorithm. 
Whenever $k<m$, restrict the action to the subset $\cA_\cZ^*$ of optimal $u < m$ arms of which each is optimal in at least one latent state, $\cA_\cZ^* = \{a\in [k] : \exists z\in \cZ : a_z^* = a\}$, and run a standard MAB algorithm restricted to $\cA_\cZ^*$. When $k\geq m$, run a standard MAB algorithm on $\cA = [k]$. This procedure achieves $O(\sqrt{\min(k, m)T})$ regret in the worst case on the latent bandit and latent preference bandit problems. 
\end{thmprop}
\begin{proof}
    The procedure uses $\min(|\cA|, |\cA_\cZ^*|) \leq \min(k, m)$ arms, and the optimal arm for any latent bandit or LPB instance is contained in either action set. We can directly apply standard regret bounds for MAB algorithms (see e.g.,~\citet{lattimore2020bandit}) to the restricted action sets.
\end{proof}

Notably, the procedure in Proposition~\ref{prop:regret_upper} does not use knowledge of the set of possible reward \emph{means}, only the set of possible best arms. On its face, it may seem that the \mts{} algorithm of \citet{hong2020latent} is matched by this simple algorithm. However, as we see empirically in Section~\ref{sec:experiments}, this is far from true: \mts{} achieves substantially better performance by exploiting reward structure, even though this is not yet explained by theory. In fact, when all reward means are distinct across states and known, $\mu_{a,z} \neq \mu_{a,z'}$, regret constant in $T$ is achievable for latent bandits. However, this is not achievable in the LPB setting since the mean parameters must be estimated during exploration.
In short, to reach optimal empirical performance on the latent preference bandit problem, and hope to match the lower bound in \eqref{eq:lower_bound}, it is critical to exploit the structure between arm parameters. 

\begin{algorithm}[t]
    \centering
    \caption{Thompson sampling for LPB regret minimization with absolute rewards (\lpbts{})}\label{alg:regret_min_abs}%
    \begin{algorithmic}[1]
        \State Let $\hat{p}(z) = \frac{1}{m}$ for latent states $z=1, ..., m$
        \For{$t=1, ...$}
        \State Sample $z_t \sim \hat{p}(z)$
        \State Perform action $a_t = a^*_{z_t}$
        \State Observe $r_t$ from the environment
        \For{$z = 1, ..., m$}
        \State Update mean parameters $\hat{\mu}_z$ according to \eqref{eq:mu_est}
        \EndFor
        \State Update latent state posterior $\hat{p}(Z)$ according to \eqref{eq:appr_posterior}
        \EndFor
    \end{algorithmic}
\end{algorithm}

\subsection{A posterior-sampling LPB algorithm exploiting the ordering of means}

We propose the \lpbts{} algorithm (Algorithm~\ref{alg:regret_min_abs}) for the LPB problem based on sampling from the posterior of the latent state and selecting the optimal arm for that state. First, let $\cD_T = ((a_1, r_1), ..., (a_T, r_T))$ denote the history of the first $T$ observations collected during exploration for a problem instance $(z, \bmu)$. The likelihood of $\cD_T$ under a state $z$ with preference ordering $O_z$ is then
$$
\cL(\cD_T \mid Z=z) = \prod_{t=1}^T p(r_t \mid a_t, z) = \int_{\bmu \in \cH_z} p(\bmu \mid z) \prod_{t=1}^T  p(r_t \mid a_t, \bmu, z) d\bmu~.
$$
and can be used to construct the posterior probability $p(Z=z \mid \cD_T)$, provided that a well-specified parameter prior $p(\bmu \mid z)$ is known for each latent state $z$. In general, the constraint $\bmu \in \cH_z$ means that no closed-form expression exists, and computing it exactly is intractable. In principle, we could appeal to variational inference~\citep{jordan1999introduction} for an approximation, and if a strong parameter prior $p(\bmu \mid z)$ is available, this can offer substantially more information to the learner than the ordering $o_{1,z} \succeq o_{2,z} \succeq ... \succeq o_{k,z}$ implied by $z$. In the extreme case that $p(\bmu \mid z)$ is a delta function, this coincides with the latent bandit problem of~\citet{hong2020latent}. As we try to minimize the information needed about the latent variable, \emph{we assume that no parameter prior is available}. 

Without a parameter prior, the likelihood $p(r_t \mid a_t, z)$ is not fully defined, but we can construct an upper bound on the likelihood by considering the mean configuration with the highest likelihood for the data restricted to the available orderings implied by $\cZ$. For all states $z$, 
\begin{align*}
\cL(\cD_T \mid Z=z) & =  \int_{\bmu \in \cH_z} p(\bmu \mid z) \prod_{t=1}^T  p(r_t \mid a_t, \bmu, z) d\bmu \leq \sup_{\bmu \in \cH_z} \prod_{t=1}^T  p(r_t \mid a_t, \mu_{a_t})~.
\end{align*}
With Gaussian rewards, maximizing this \emph{upper} bound corresponds to minimizing the mean squared error of $\bmu$ in predicting the observed reward, constrained to the set $\cH_z$. Thus, under the assumption that $z$ is the correct latent state, we may estimate the mean parameters as follows. 
\begin{equation}\label{eq:mu_est}
\hbmu_z \coloneqq \argmin_{\bmu \in \cH_z} -\ell(\cD_T \mid \bmu), \quad \mbox{where} \quad
-\ell(\cD_T \mid \bmu) \propto \sum_{t=1}^T \frac{(r_t - \mu_{a_t})^2}{\sigma^2}
\end{equation}
With $\{\hbmu_z\}$ the minimizers of \eqref{eq:mu_est} for all $z$, we can construct an \emph{optimistic} posterior estimate, 
\begin{equation}\label{eq:appr_posterior}
\forall z : \hat{p}(z\mid \cD_t) \coloneqq \frac{1}{\alpha} p(\cD_t \mid \hbmu_z)~.
\end{equation}
where $\alpha$ is the normalization constant.
Note that this is an approximation as the bound $p(\cD_t \mid \hbmu_z) \geq p(\cD_t \mid z)$ affects also the normalization constant of the posteriors: The looser the bound for one state, the lower the posterior probability of other states.

We design our method, \lpbts{} (Algorithm~\ref{alg:regret_min_abs}) for regret minimization by selecting the optimal arm for a state sampled from the approximate posterior \eqref{eq:appr_posterior}. 
The constrained maximum-likelihood estimation (MLE) problem in \eqref{eq:mu_est} solved for each state is a quadratic program with linear inequality constraints. We show below that it can be solved using off-the-shelf solvers for isotonic regression~\citep{barlow1972isotonic}. We discuss the computational complexity of \lpbTS{} in the Appendix.

\begin{thmprop}\label{thmprop:isotonic}
    Let $n_a = \sum_{t=1}^T\mathds{1}[a_t = a]$ and define $\bar{w}_a = \frac{n_a}{\sigma^2}$. Next, let $O_z = (o_1, ..., o_k)$ be the preference ordering of latent state $z$. Then, the solution to the isotonic regression problem with outcomes $y_a = \frac{1}{n_a}\sum_{t : a_t=a}r_t$ and sample weights $\bar{w}_a$ 
    \begin{equation*}
    \begin{aligned}
    & \underset{\bmu \in \bbR^d}{\text{minimize}}
    & & \sum_{a=1}^k \bar{w}_a (\mu_a - y_a)^2 
    & \text{subject to}
    & & \mu_{o_k} \leq \mu_{o_{k-1}} \leq ... \leq \mu_{o_1}
    \end{aligned}
    \end{equation*}
    solves the constrained MLE problem in \eqref{eq:mu_est}. A proof is given in Appendix~\ref{app:isotonic}.
\end{thmprop}

\subsubsection{On the scaling of regret with latent states, actions and constraints}
\label{sec:scaling}

In Section~\ref{sec:regret_min}, we gave a simple procedure that achieves $O(\sqrt{\min(k,m)T})$ regret on the latent preference bandit problem by restricting the action set to the $\leq m$ arms that are optimal in at least one latent state. 
Although we don't prove this formally, we should expect a similar regret bound to be provable for \lpbts{}, by appealing to general results~\citep{geyer1994asymptotics} which state that constrained MLE has asymptotic variance that is no worse than unconstrained MLE (as used in many MAB algorithms), provided that the true parameter lies in the constraint set, which it does here. 
Indeed, in our experiments, both \lpbts{} and \mts{}~\citep{hong2020latent} vastly outperform an uninformed MAB algorithm, even when restricting the number of arms as described above. This gap is not predicted by worst-case upper bounds for \mts{} and MAB algorithms. 

The empirical performance of \lpbts{} is affected by the nature and number of order constraints. For $k$ arms, there are $k!$ possible orderings. Thus, even when $m$ grows linearly with $k$, the space of possible parameter vectors is reduced by an $O(k!)$ factor compared to an unconstrained problem. The effect of this is largest when many of the constraints in \eqref{eq:mu_est} are active for states $z'$ that are not the ground-truth state $z^*$, and their MLE estimate projects the empirical reward means onto $\cH_{z'}$. The more active constraints, the larger the projection and the smaller the likelihood of observed rewards under $z'$. 

As a toy example, assuming that orderings $O_z$ are randomly selected from all possible permutations of the $k$ actions without replacement, the probability that two such orderings differ in only two positions is $\binom{k}{2}/(k! - 1)$ (see Appendix~\ref{app:similarity_orderings}). 
For large $k$, this probability is vanishingly small. For example, with $k = 10$, this probability evaluates to approximately $\frac{45}{3,628,799} \approx 1.24 \times 10^{-5}$. Given that $m$ is typically much smaller than $k!$, the ground-truth state will stand out even more as $k$ grows. This explains why subsampling the action set as in Section~\ref{sec:regret_min} may achieve surprisingly low worst-case regret but performs substantially worse empirically than algorithms exploiting latent structure. 

%
%

%
%
\section{Experiments}
\label{sec:experiments}
We compare our algorithm \lpbTS{}, on both synthetic and realistic tasks, to: i) standard MAB \textbf{Thompson Sampling (TS)}~\citep{Thompson1933, russo2018tutorial} that is oblivious to the latent state structure, initialized with Gaussian priors and using the ground-truth variance of the reward, ii) \textbf{TS, top arm subset} (see Proposition \ref{prop:regret_upper}), and iii) \textbf{\mts{}} ~\citep{hong2020latent}, Thompson Sampling with a perfect latent variable model, including the exact reward mean vectors for all arms and latent states.

\subsection{Synthetic Experiments}
We first consider a synthetic bandit environment designed to simulate a well-specified LPB setting. The environment is parameterized by fixable variable numbers of $k$ arms and $m$ latent states. For given $k$ and $m$, for each latent state $z \in [m]$, a random preference ordering (permutation) $O_z = (o_{z,1}, ..., o_{z,k})$ of the arms is generated without replacement. Rewards are Gaussian distributed, with variance $\sigma^2 = 1$. The mean reward for arm $a$ in state $z$, $\mu_{a,z}$ is based on its position in $O_z$, ensuring a strictly decreasing sequence of means along the permutation, with the optimal arm in each state, $o_{z,1}$ assigned mean $\mu_{o_{z,1},z} \sim \mathcal{U}(C + \Delta k, C + \Delta k + 1)$. We use $C=9$, and $\Delta = 0.2$. For subsequent arms $o_{z,j}$,  $j \in \{2, \ldots, k\}$, the mean is defined recursively as $\mu_{o_{z,j},z} = \mu_{o_{z,j-1},z} - \Delta - U$, where $U \sim \mathcal{U}(0, \epsilon)$ with $\epsilon = 0.05$. We also allow for reward scales to vary among instances in the same latent state, by having a larger draw interval for the mean reward of the optimal arm: $\bar{\mu}_{o_{z,1},z} \sim \mathcal{U}(\bar{C} + \Delta k,~\bar{C}+\Delta k + \gamma k)$ with $\bar{C} =6 $ and $\gamma=0.4$, and subsequent arm means defined recursively as before. We conduct experiments with $N=50$ independent samples of instance means per latent state, each having $T=200$ rounds. We report all the configurations of $k, m, T~\text{and}~N$ in the results presented. We use cumulative regret over $T$, and average regret at final round $T$ as the error metrics. We report these as averages over $m \times N$ independent instances, and their corresponding standard deviation as errors. We also report the average active constraints in \lpbTS{}, reported at the final round T for each value of the varied parameter, with standard deviation as the error.

\subsection{MovieLens Experiments}
For a realistic personalisation setting, we construct a task based on the MovieLens~\citep{harper2015movielens} datasets where actions represent movie choices, rewards are ratings of movies, and latent states are groups of users. See Appendix~\ref{app:movielens} for a full description. First, users are clustered into latent states based on their real-world sparse movie preference ratings. Then, preference orderings for all clusters are learned using Bradley-Terry Models (BTMs)~\citep{bradley1952rank} and used to define reward models assigning an average rating in $\mu_{a,z} \in [1, 5]$ to each movie. Personalized rating scales for each instance are created by drawing a random rating interval defined by the smallest allowable interval length, $\zeta=1.5$. For each experiment, we sample 100 random users and do 200 rounds of movie ratings per user, where at each round, 300 genre-diverse movies are sampled from the set of available movies to form the active action set.
\begin{figure}
    \centering
    \includegraphics[width=1\textwidth]{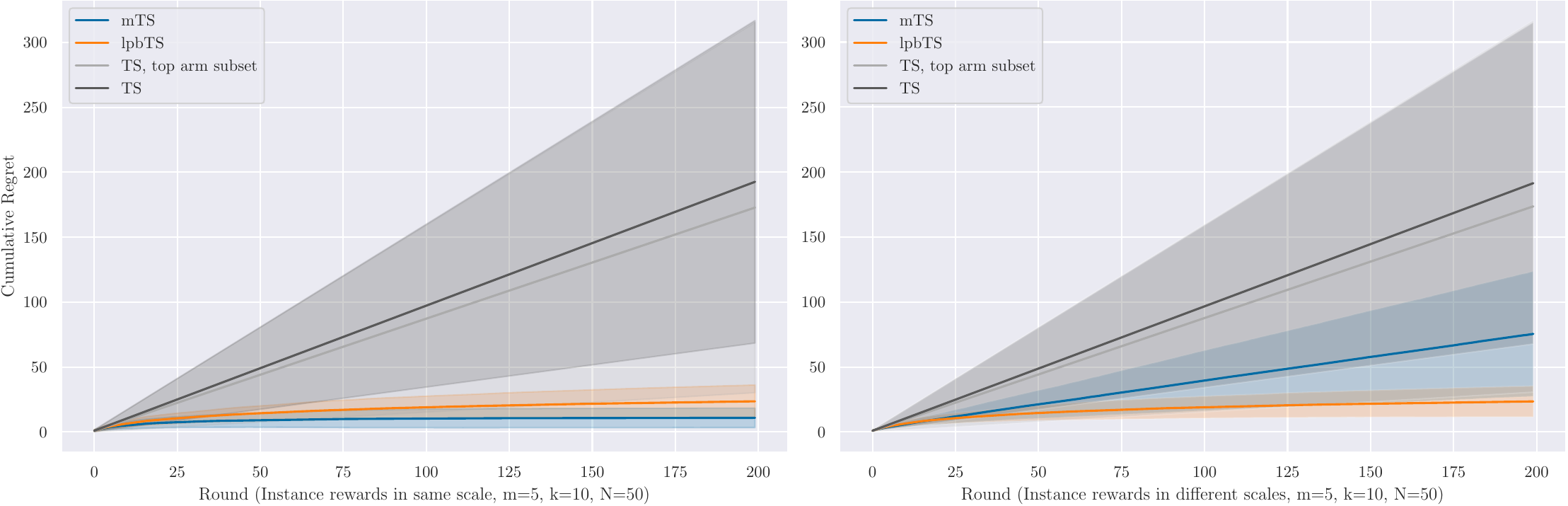}
    \caption{Synthetic experiment, cumulative regret compared to baselines ($m=5, K=10, N=50$). \lpbTS{~\textbf{(Ours)}} is comparable to latent bandit baselines when instance rewards lie in the same scale (\textbf{\textit{Left}}) and outperforms baselines when instance rewards in different reward scales (\textbf{\textit{Right}})  }.
    \label{fig:cumulative_rewards_all}
\end{figure}
We evaluate how well \lpbts{} performs compared to baselines i) when movie ratings for users of the same latent state are in the same absolute scale, ii) when individual ratings could vary in scale for users, and iii) when using a \emph{learned} latent preference ordering model $\{(z, \hat{O}_z)\}$ (see Appendix~\ref{app:movielens}) rather than the ground-truth model, to study the effects of possible errors in model fit and latent state recovery. Cumulative regret results are reported as averages and standard deviations over user instances. Additionally, we report the average of movie ratings achieved by the algorithms, standardized with $z$-score standardization \textit{per user} to ensure consistency when rewards can vary in scale, together with standard errors.

\subsection{Results}

\paragraph{LPB structure benefits exploration, and is robust}

In Figure~\ref{fig:cumulative_rewards_all}~(\textit{Left}), we see that \lpbTS~is comparable in performance to \mts{}, in the comparable convergence of the cumulative regret, when instance reward means have a fixed reward scale in the latent states. However, \mts{} outperforms \lpbTS~as it converges quicker, with a lower variance because of it's knowledge of the \textit{true} latent reward means. \lpbTS~has to \textit{estimate} these latent means with noisy rewards. Adding the ordering constraints $O$ is vastly beneficial compared to no stucture as in TS. Unsurprisingly, TS, top arm subset is also poor, because it does not actually exploit structure despite achieving the upper bound. In Figure~\ref{fig:cumulative_rewards_all} (\textit{Right}), we see the benefit of using a more general latent structure $O$ compared to latent mean vector of rewards: Cumulative regret for \lpbTS~converges quicker than that for \mts{} with a visibly lower variance. This is because the latent model comprising mean vectors is misspecified when absolute reward scales can vary for different latent state instances.

\begin{figure}[t]
    \centering
    \begin{subfigure}[b]{0.85\textwidth}
        \centering
        \includegraphics[width=\textwidth]{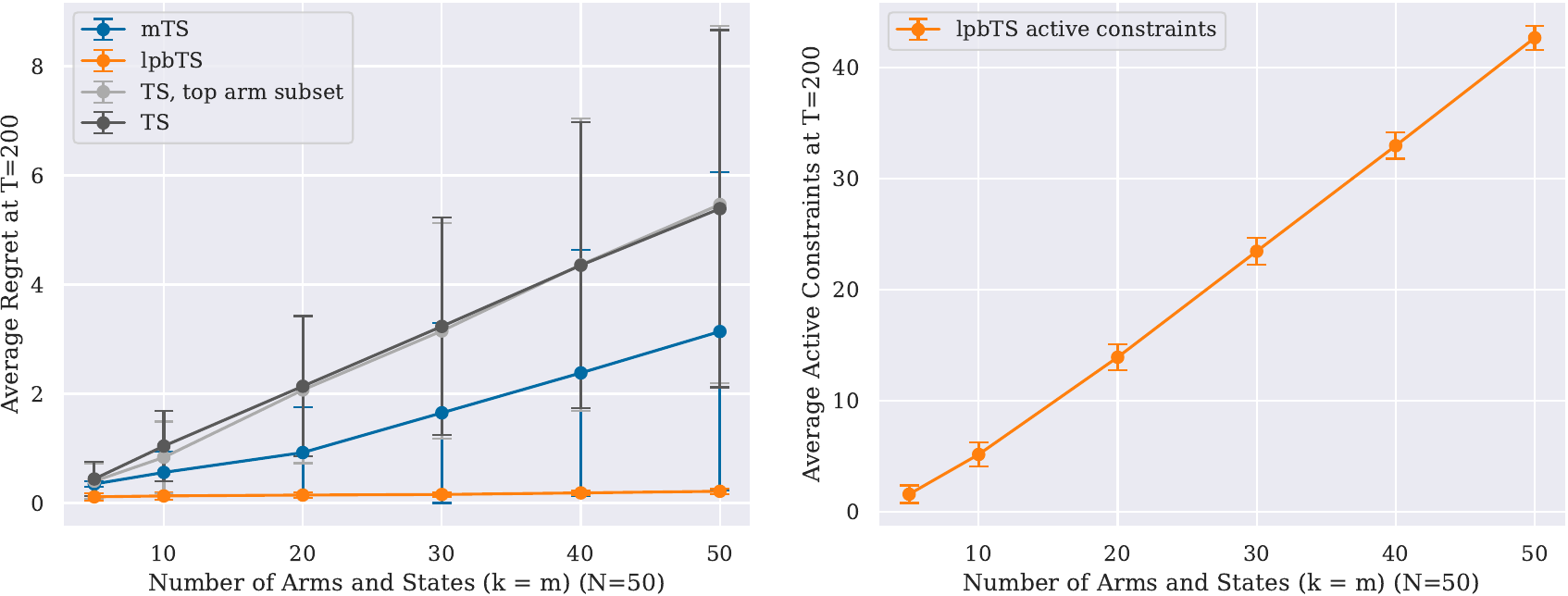}
        \caption{ Varying the number of arms $k$, and $m=k~(N=50, T=200,  k\in[5, 10, 20, 30, 40, 50]$). \textbf{\textit{Left:}}  Observed average regret at $T=200$. \textbf{\textit{Right:}} Observed average active constraints.}
        \label{fig:active_k_m}
    \end{subfigure}
    \begin{subfigure}[b]{0.85\textwidth}
        \centering
        \includegraphics[width=\textwidth]{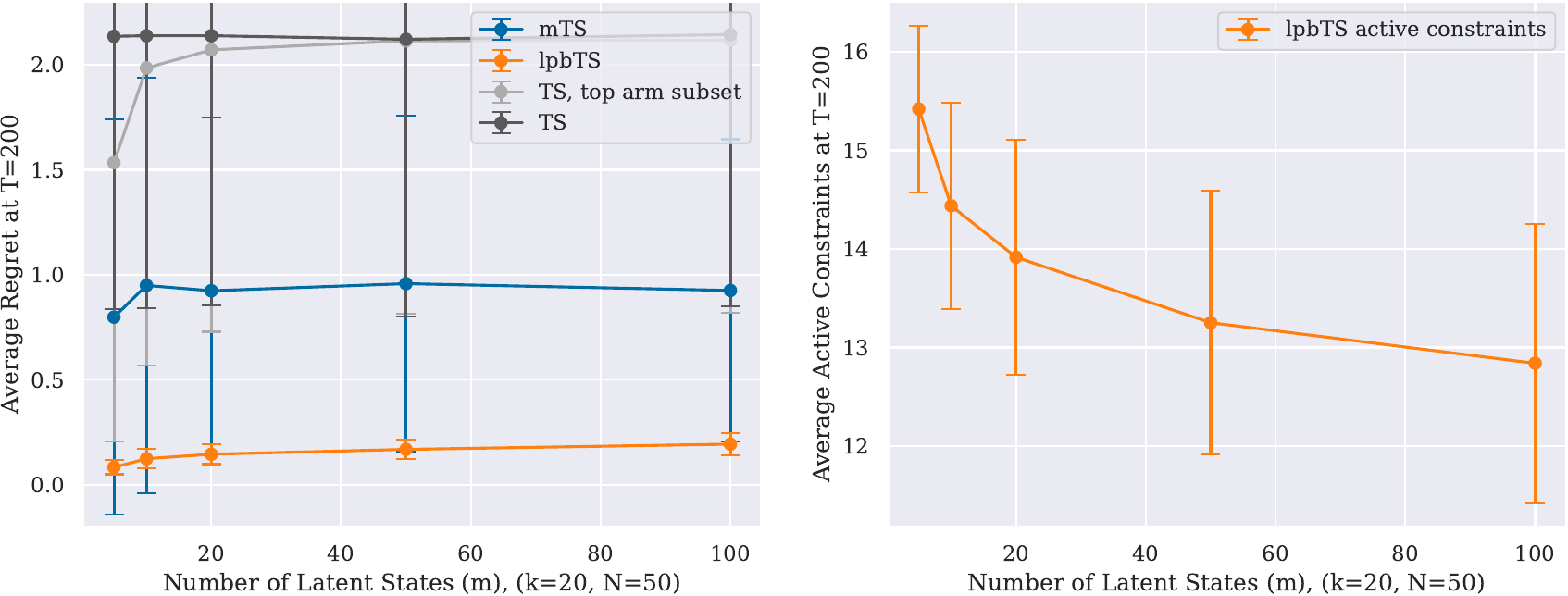}
        \caption{Varying the number of latent states $m$ ($k=20, N=50,  T=200, m\in[5, 10, 20, 50, 100]$). \textbf{\textit{Left:}} Observed average regret at $T=200$ \textbf{\textit{Right:}} Observed average active constraints at $T=200$. }
        \label{fig:varying_m}
    \end{subfigure}
    \hfill
    \caption{Synthetic experiment, instance rewards in different scales.  When $m=k$ and $k$ grows, the true latent state stands out more as its closest neighbors become more different with high probability. When $m$ increases with $k$ fixed, it is possible to assign more distinct means and orderings, reducing the structural benefit of latent bandits (\lpbTS{}, \mts{}) over  unconstrained methods (TS). 
    }
    \label{fig:varying_combined1}
\end{figure}

\paragraph{LPB characteristics become explainable from observed active constraints} In Figure~\ref{fig:active_k_m}, we observe that as $k$ and $m=k$ increase, the average regret for \mts{} also increases while that of \lpbTS~is fairly constant. This is because instance rewards have different scales, and the draw interval for the mean reward of the optimal arm grows with $\Delta k$. If the interval was fixed, \mts{} would still be worse than \lpbTS, but would not grow this way (e.g. see Figure~\ref{fig:varying_m}). We also see that the number of active constraints grow. This is because $m=O(k)$, but the number of possible permutations grows like $k!$, so the probability of having large differences between states grows when $m=k$ and $k$ grows. This is not predicted by a $O(\sqrt{\min(k, m)T})$ bound since $m=k$. It is explained by the fact that the true latent state stands out more 
with high probability, and the empirical isotonic means $\hat{\mu}_z$ become less likely to align with the neighbor states (the most confusable states) relative to the true state, resulting in a higher number of active constraints. In Figure~\ref{fig:varying_m}, we see that when the number of latent states is increased with $k$ fixed, the average regret increases in \mts{}, \lpbTS, and TS, top arm subset until $m=k$ at which point it flattens. This can be explained by the worst-case $O(\sqrt{\min(k, m)T})$ bound.
We defer additional empirical results and discussions aligning with these insights to Appendix~\ref{app:additional_experiments}, in Figures~\ref{fig:app_active_D}-\ref{fig:app_active_sigma} for varying $\Delta, k (\text{with m fixed), and}~\sigma$, and Figures~\ref{fig:app_active_k_m}-\ref{fig:app_active_k_s} for instance rewards in the same scale.

\paragraph{MovieLens results are aligned}
Results from MovieLens experiments are consistent with the theory and synthetic results, where \lpbTS{} is comparable (Figure~\ref{fig:movielens_same_scale_20m}, 20M) to \mts{} (which uses the oracle ground truth model) when ratings are in the same scale for the latent states, and \mts{} is worse (Figure~\ref{fig:movielens_diff_scale_20m}, 20M) when ratings are in different scales. Further, \lpbTS{} with an offline recovered model is competitive compared to \lpbts, oracle ground-truth model, albeit biased especially when user ratings are in different scales. We re-iterate that recovering latent models is a key ingredient for latent bandits, and we empirically demonstrate a method to recover them for Latent Preference Bandits with our two-stage recovery, with many actions. Results for 1M and 32M datasets are left to Appendix~\ref{app:additional_experiments}.

\begin{figure}[t]
    \centering
    \begin{subfigure}[b]{0.85\textwidth}
        \centering
        \includegraphics[width=\textwidth]{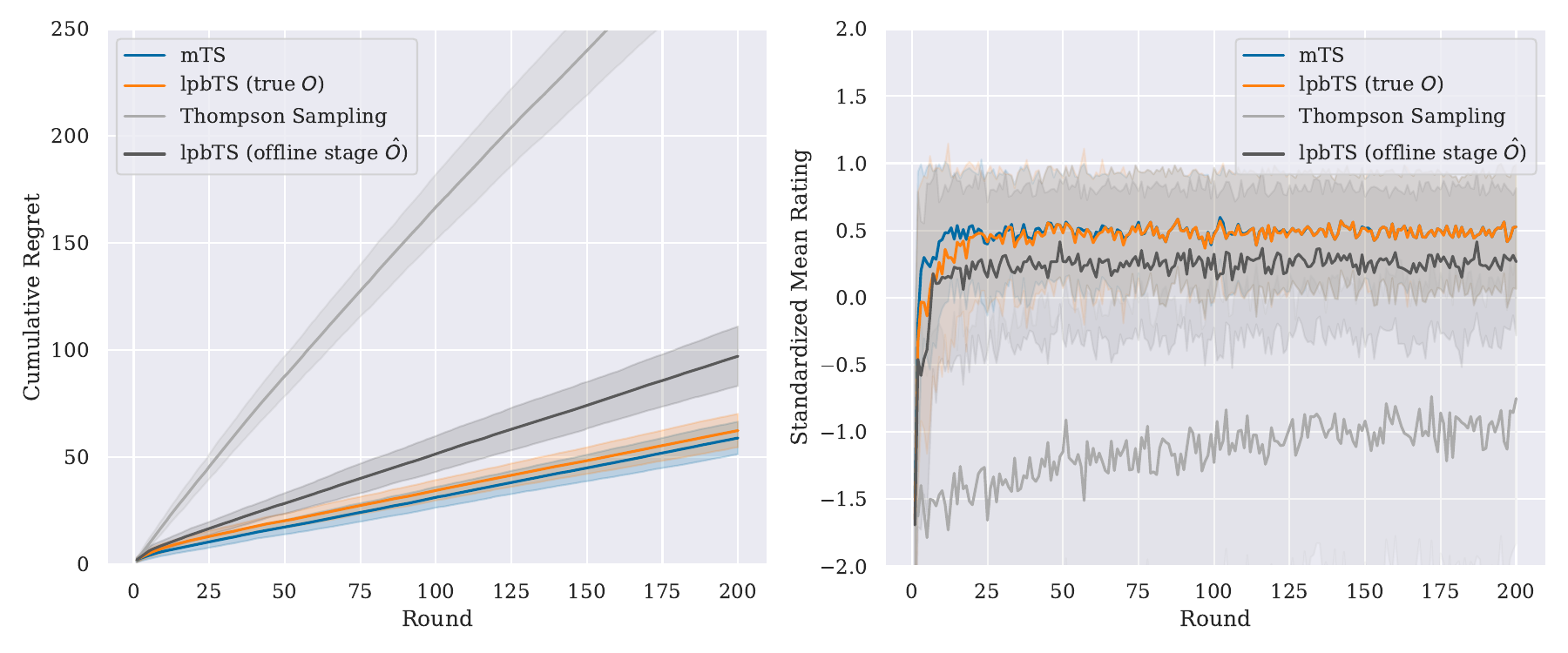}
        \caption{Movie ratings in the same scale for users in a latent state}
        \label{fig:movielens_same_scale_20m}
    \end{subfigure}
    \hfill
    \begin{subfigure}[b]{0.85\textwidth}
        \centering
        \includegraphics[width=\textwidth]{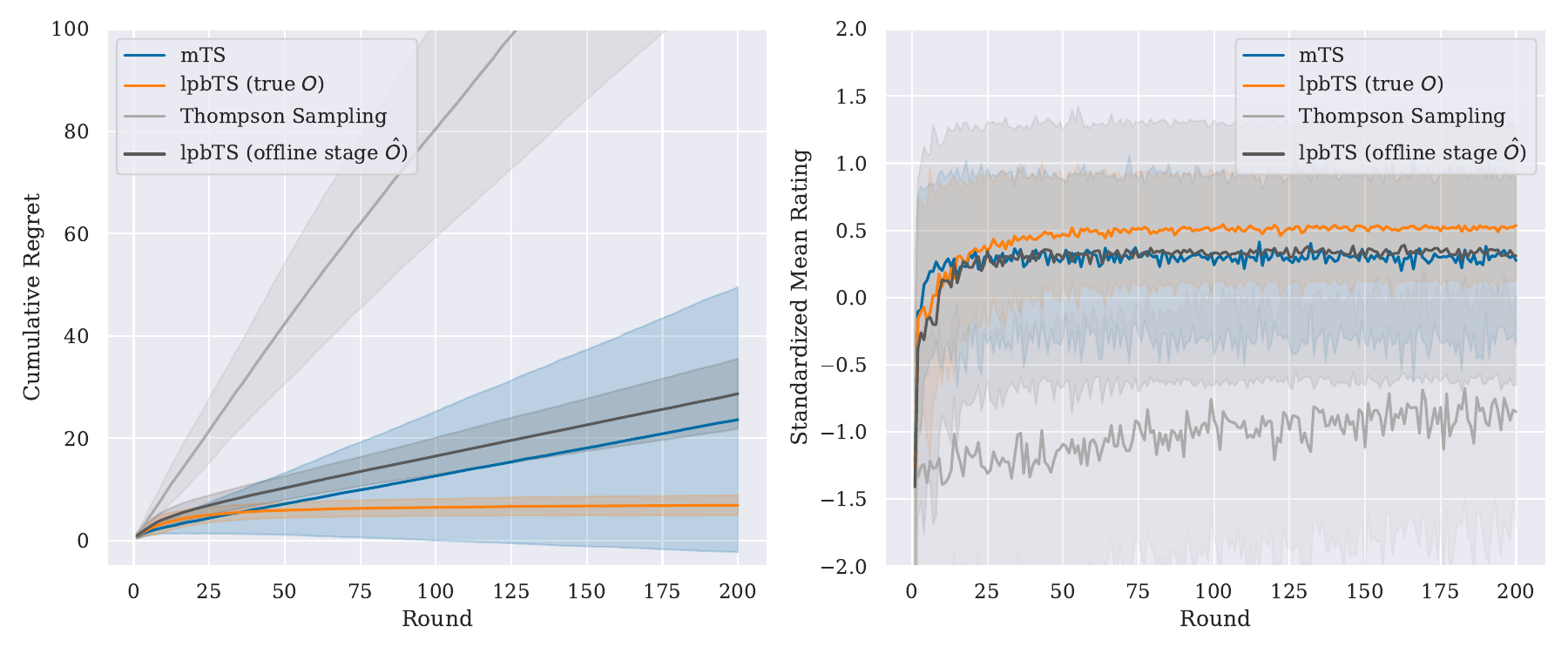}
        \caption{Movie ratings in different scales for users in a latent state}
        \label{fig:movielens_diff_scale_20m}
    \end{subfigure}
    \caption{MovieLens Experiment, 20M Dataset. Results match theory: \lpbTS{~\textbf{(Ours)}} is comparable to mTS in (a), outperforms in (b), and the two-stage recovery of $O$ is empirically validated. }
    \label{fig:movielens_combined}
\end{figure}

\section{Conclusion}
We have proposed Latent Preference Bandits (LPB), a novel bandit problem where instances share structure based on a latent variable that determines the preference ordering of actions but not their rewards. We design the algorithm \lpbTS{} for regret minimization in this setting by sampling from an approximate posterior of the latent state, constrained by the set of possible orderings. We demonstrate that, despite using less information, \lpbTS{} is competitive with latent bandit algorithms that have full knowledge of the reward distribution of each arm when all instances of the same state have the same distribution, and outperforms them when individual rating scales differ between instances who share the same preference ordering. The benefit over uninformed bandit algorithms grows when the number of latent states (orderings) is small relative to the number of arms. A limitation of our algorithm is that it requires a model of the preference ordering of all latent states. However, in our experiments on movie ratings, we find that a learned model performs comparably to the ground truth. 

\bibliographystyle{plainnat}
\bibliography{main}

\clearpage
\appendix

\section{Notation}
A list of common notations is given in Table~\ref{tab:notation}. Generally, capital Roman letters denote random variables and lower-case Roman letters denote constants or observations or random variables. The sequence $O_z$ is an exception since it is not random. Caligraphic Roman letters denote sets. 
\begin{table}[H]
    \centering
    \caption{Common notation}
    \label{tab:notation}   
    \vspace{0.5em}
    \begin{tabular}{ll}
    \toprule
        $s, t$ & Time indices \\
        $T$ & Time horizon \\
        $\cA$ & Set of available actions \\
        $a$ & A single action in $\cA$. $a_t$ is the action at time $t$\\
        $\cZ$ & Set of latent states \\
        $Z$ & Latent variable on $\cZ$ \\
        $z$ & A single latent state in $\cZ$ \\
        $R_a$ & Stochastic reward for action $a$ \\
        $r_t$ & Observation of reward at time $t$ following action $a_t$ \\
        $\mu_a$ & Expected reward under action $a$, $\mu_a = \E[R_a]$ \\
        $\sigma_a$ & Standard deviation of reward under action $a$, $\sigma^2_a = \V[R_a]$ \\
        $\bmu$ & Vector of reward means for all actions \\
        $a^*$ & Action with the highest expected reward  \\
        $\mu^*$ & Optimal reward, reward of optimal action \\
        $\mu_{a,z}$ & Expected reward under action $a$ in latent state $z$ \\
        $\sigma_a$ & Standard deviation of reward under action $a$ in latent state $z$\\
        $\bmu_z$ & Vector of reward means in latent state $z$ for all actions \\
        $\cH_z$ & Set of valid reward means in latent state $z$, $\bmu_z \in \cH_z$ \\
        $\cH$ & Set of globally valid reward means, $\bmu \in \cH = \cup_{z \in \cZ} \cH_z$\\
        $a_z^*$ & Action with the highest expected reward in latent state $z$ \\
        $\mu_z^*$ & Optimal reward for any action in latent state $z$ \\
        $\hmu_a, \hbmu, \hmu_{a,z}, \hbmu_z$ & Estimates of reward means corresponding to the above \\
        $\pi$ & Decision making policy to select $a_t$ \\
        $\reg(T)$ & Cumulative regret until horizon $T$ \\
        $\cM$ & Latent-variable model \\
        $O_z = (o_{z,1}, ..., o_{z,k})$ & Preference ordering of actions for latent state $z$ \\
        $i_{a,z}$ & Rank of action $a$ under latent state $z$ \\
        $\rho$ & Separation parameter for preference probability \\
        $\Delta$ & Separation parameter for reward means \\
        $\tilde{R}_t$ & Preference feedback at time $t$ \\
        $\cD_T$ & History of observations up to time $T$ \\
        $w_a$ & Weight parameter in isotonic regression \\
    \bottomrule
    \end{tabular}
\end{table}

\newpage
\section{Details on the MovieLens experiments}
\label{app:movielens}
For a real-world personalisation setting, we do experiments on the MovieLens~\citep{harper2015movielens} datasets, using the 1M, 20M and 32M versions. For each dataset, we first filter out movies that have less than 200 ratings, and users who have rated less than 200 movies. The resulting sizes are \textit{1M: }1,589 users, 1,132 movies, 498,677 ratings,\textit{ 20M:} 26,826 users 6,236 movies, 11,905,303 ratings, \textit{32M:} 42,197 users 8,777 movies, 19,832,758 ratings. We use $k= \text{\textit{\#movies}}$ after filtering.  We then obtain a train-test split on users with a split ratio of 0.5, and further split the test user set into a bandit inference set, and an \textit{``offline estimation''} set with a split ratio of 0.5. To get the ground-truth latent preference orderings $O$, a two-stage preference ordering recovery proceeds as follows: i) KMeans clustering on the training set, with $5$ clusters ($m=5$) on the sparse ratings. ii) For all users assigned to a cluster $z$, the cluster-specific preference orderings for all movies are obtained by fitting preference orderings on the users' movie ratings, with Bradley-Terry Models (BTMs)~\citep{bradley1952rank}, see  Appendix~\ref{app:latent_recovery}. 

Our two-stage BTM-based recovery yields a set of $\beta_{z,a}\in[0, 1], z\in[m], a\in[k]$ where $\beta_{z,a}$ define the preference strength for movie $a$ in cluster $z$ after sigmoid transformation, and $O_z$ is the indices of the decreasing sort of $\beta_z$. This two-stage preference ordering recovery is motivated by our observation that the recovery is comparable in performance to a BTM Expectation-Maximization (EM) approach for latent preference recovery, like the one proposed by~\citet{vigneau1999analysis}, see Appendix Figure~\ref{fig:comparison_recovery}. For a user at bandit time, their true $z$ (obtained, using the KMeans model fit in the two stage recovery) is used to get the environment reward. The environment reward for a user in latent state $z$, and for movie $k$ is obtained as $\mathcal{N}(\bar{\beta}_{z,k}, 0.5)$, where $\bar{\beta}_z$ is a scaling of $\beta_z$ to the MovieLens rating scale [1, 5] using $\bar{\beta}_{z} =1+4\times \frac{\beta_z - \min(\beta_z)}{\max(\beta_z) - \min(\beta_z)}$. We also allow for instance rewards to vary in scale to $\bar{\bar{\beta}}_z$ computed as: $\bar{\bar{\beta}}_z = \eta + (\upsilon - \eta) \times \frac{\bar{\beta}_{z} - 1}{5 - 1}$
where $\eta~\text{is their possible minimum rating with}~\eta \sim \mathcal{U}(1, 5 - \zeta)$, and $\upsilon$ is their possible maximum rating with $\upsilon \sim \mathcal{U}(\eta + \zeta, 5)$ and $\zeta$ is the allowable minimum rating interval with $\zeta = 1.5$. For varying individual scales, we only consider reward randomness due to the random rating perturbation.

For each experiment, we sample 100 users and do 200 rounds per user, where at each round, 300 genre-diverse movies are sampled from the set of available movies. In our experiments, we aim to compare how well our \lpbts{} algorithm performs compared to baselines for the following: i) Settings where movie ratings for users belonging to the same latent state are in the same absolute scale, and when individual ratings could vary in scale for users. ii) How well our algorithm performs with an offline recovered latent preference ordering model, that could have recovery errors. To recover this latent preference ordering model, we use the \textit{``offline estimation''} data, with uniformly collected movie rating logs between 200 and 300 per user, for both settings in i). This also provides insights into the quality of the two-stage latent preference order recovery.
Cumulative regret results are reported as averages over user runs, with standard deviation errors. The average movie ratings over users is also reported. To ensure consistency when rating scales vary, the average ratings are standardized with $z$-score standardization \textit{per user} and errors reported as standard errors of the mean.

\newpage
\section{Proof that isotonic regression solves the constrained MLE problem}
\label{app:isotonic}
\begin{thmprop}
    Let $n_a = \sum_{t=1}^T\mathds{1}[a_t = a]$ and define $\bar{w}_a = \frac{n_a}{\sigma^2}$. Next, let $O_z = (o_1, ..., o_k)$ be the preference ordering of latent state $z$. Then, the solution to the isotonic regression problem with outcomes $y_a = \frac{1}{n_a}\sum_{t : a_t=a}r_t$ and sample weights $\bar{w}_a$ 
    \begin{equation*}
    \begin{aligned}
    & \underset{\bmu \in \bbR^d}{\text{minimize}}
    & & \sum_{a=1}^k \bar{w}_a (\mu_a - y_a)^2 
    & \text{subject to}
    & & \mu_{o_k} \leq \mu_{o_{k-1}} \leq ... \leq \mu_{o_1}
    \end{aligned}
    \end{equation*}
    solves the constrained MLE problem in \eqref{eq:mu_est}.
\end{thmprop}
\begin{proof}
    We show that the isotonic regression objective is equal to \eqref{eq:mu_est} up to a constant. We have 
    \begin{align*}
        \sum_{a=1}^k \bar{w}_a (\mu_a - y_a)^2 & = \sum_{a=1}^k \bar{w}_a (\mu_a^2 - 2y_a\mu_a + y_a^2) 
         = \sum_{a=1}^k \bar{w}_a (\frac{1}{n_a}\sum_{t:a_t=a}[\mu_a^2 - 2r_t\mu_a] + y_a^2) \\
        & = \sum_{a=1}^k \frac{1}{\sigma_a^2} (\sum_{t:a_t=a}[\mu_a^2 - 2r_t\mu_a + r_t^2] + y_a^2 - \sum_{t:a_t=a}\frac{r_t^2}{n_a}) 
         = \sum_{t=1}^T \frac{(\mu_{a_t}^2 - r_t^2)}{\sigma_{a_t}^2} + C~,
    \end{align*}
    where $C$ is a constant w.r.t. $\bmu$. Thus minimizing the LHS and RHS yields the same solution. 
\end{proof}

\section{A note on relative feedback}
\label{app:relative}

Dueling bandits~\citep{yue2009interactively, sui2018advancements, bengs2021preference} are the simplest-to-analyze as rewards are observed in the same format as the latent state--as relative preference feedback. Let $\mathcal{D}_T = ((a_t, a'_t, \tilde{r}_t))_{t=1}^T$ be a sequence of preference feedback events where $a_t, a'_t \in [m]$ are two competing actions and $\tilde{r}_t \in \{0,1\}$ indicates which action was preferred. Assuming that there are latent, noisy and continuous rewards $r_t, r'_t$ for the two actions, let $\tilde{R}_t = \mathds{1}[R_t \geq R'_t]$ indicate noisy preferences for $a_t$ over $a'_t$ and $\tilde{r}_t$ its realization. 

If for any latent state $z \in [m]$, there exists an (unknown) margin parameter $\rho > 0$ such that most of the time, with a margin $\rho$, the reward for $a$ is higher than the reward for $a'$, if $a$ is preferred to $a'$ 
then, 
$$
p(\tilde{R}_t = 1 \mid a_t, a'_t, z) \leq \left\{
\begin{array}{ll}
\frac 1 2 -\rho, & a_t \succeq_z a_t' \\
1, & \mbox{otherwise}
\end{array}\right.
$$
provides a crude upper bound on the likelihood of a single reward from $\mathcal{D}_T$ under $z$. In other words, the probability of observing $r_t > r'_t$ is less than $1/2-\rho$ if the action $a'_t$ has lower rank than $a_t$. Defining $n_i(z)$ to be the number of observed inversions of the rank imposed by $z$
$
n_i(z) = \sum_{t=1}^T (\mathds{1}[I_{a_t}(z) > I_{a'_t}(z)] \neq \mathds{1}[r_t \leq r'_t])~
$
then, we can upper bound the full likelihood under $z$ as 
\begin{equation}\label{eq:rel_lik_bound}
p(\mathcal{D}\mid z) \leq (\frac 1 2 -\rho)^{n_i(z)}\cdot 1^{T-n_i(z)} \leq 2^{-n_i(z)}
\end{equation}
and, likewise, the posterior $p(z \mid \mathcal{D}) \propto p(\mathcal{D}\mid z)p(z)$.

Bounding the posterior as in \eqref{eq:rel_lik_bound} allows us to rule out candidate latent states as the probability decays with the number of inversions, and a similar posterior sampling algorithm like \lpbTS{} for preference feedback can be obtained.

\section{Absolute to relative feedback}
\label{app:absolute_feedback}
Absolute reward feedback can always be turned into preference feedback. For example, plays and rewards $(a_1, r_1), (a_2, r_2), (a_3, r_3), (a_4, r_4)$ can be paired up consecutively: $(a_1, a_2, \mathds{1}[r_1>r_2]), (a_3, a_4, \mathds{1}[r3>r4])$. This ensures that different pairs comprise independent events, unlike, say, an all-pairs comparison. However, this procedure is likely very inefficient, statistically. Rewards obtained for each action provide much more information than is used by considering the pairs of most recent actions. Moreover, the algorithm does not rule out playing the same action twice, which means it is prone to getting stuck in local optima.

\section{Rarity of Similar Preference Orderings in LPB}
\label{app:similarity_orderings}
\begin{observation} \textbf{Rarity of Similar Preference Orderings}.
Consider $k$ actions in the LPB framework, with $O_{z_1}$ and $O_{z_2}$ as two distinct preference orderings drawn randomly from the set of all permutations of $k$ actions. The probability that $O_{z_1}$ and $O_{z_2}$ differ in exactly two positions is:
$$
P(|\{i \mid O_{z_1}(i) \neq O_{z_2}(i)\}| = 2, i\in[k]) = \frac{\binom{k}{2}}{k! - 1}
$$
\end{observation}
\begin{proof}
Consider two distinct random permutations $O_{z_1}$ and $O_{z_2}$ of $k$ actions. We need to find the probability that they differ in exactly two positions, i.e., $|\{i \mid O_{z_1}(i) \neq O_{z_2}(i)\}| = 2, i\in[k]$. Define the relative permutation $\tau = O_{z_2}^{-1} \circ O_{z_1}$. The differing positions are the non-fixed points of $\tau$ (where $\tau(i) \neq i$), so $\tau$ must be a \textit{transposition}, swapping two elements. The number of possible transpositions is $\binom{k}{2}$. The total number of ordered pairs $(O_{z_1}, O_{z_2})$ with $O_{z_1} \neq O_{z_2}$ is $k! \cdot (k! - 1)$. For each transposition $\tau$, there are $k!$ pairs where $O_{z_2} = O_{z_1} \circ \tau$, since $O_{z_1}$ can be any permutation. Thus, the number of favorable pairs is $\binom{k}{2} \cdot k!$.

Therefore, the probability is:
$$
\frac{\binom{k}{2} \cdot k!}{k! \cdot (k! - 1)} = \frac{\binom{k}{2}}{k! - 1}
$$

\end{proof}

\newpage
\section{Expanded Algorithm: \lpbTS{}~(Thompson Sampling for Latent Preference Bandits)}
We provide an expanded algorithm for \lpbTS{} in Algorithm~\ref{alg:lpbts_extended}. 
\begin{algorithm}
\caption{: \lpbTS{}~(Thompson Sampling for Latent Preference Bandits)}\label{alg:lpbts_extended}
\begin{algorithmic}[1]
\State \textbf{Input:} Number of arms $k$, latent states $m$, ordering matrix $O \in \mathbb{R}^{m \times k}$, noise $\sigma$
\State \textbf{Initialize:} For each arm $a \in [k]$, set $N_a = 0$, $S_a = 0$, $\hat{\mu}_a = 0$; for all $z \in [m]$, set $\log P_1(z) = -\log m$; set $t \gets 0$
\While{True}
    \If{$t = 0$}
        \State Select $A_t \sim \text{Uniform}([k])$
    \Else
        \State Compute normalized posterior:
        $$
        P_t(z) = \frac{\exp(\log P_t(z))}{\sum_{z'=1}^{m} \exp(\log P_t(z'))}.
        $$
        \State Sample $B_t \sim P_t(z)$
        \State Select $A_t = O[B_t,0]$
    \EndIf
    \State Observe reward $R_t$
    \State Update: $N_{A_t} \leftarrow N_{A_t}+1,\; S_{A_t} \leftarrow S_{A_t}+R_t$
    \State Set $\hat{\mu}_{A_t} = \frac{S_{A_t}}{N_{A_t}}$ (if $N_{A_t}>0$)
    \For{each $z \in [m]$}\Comment{Proposition~\ref{thmprop:isotonic}}
        \State Perform isotonic regression on the sequence $\{\hat{\mu}_{O[z,0]},\dots,\hat{\mu}_{O[z,k-1]}\}$ 
        with weights $\{N_{O[z,0]},\dots,N_{O[z,k-1]}\}$, obtaining estimates $\hat{\mu}_z[O[z,i]]$ for $i=0,\dots,k-1$. 
    \EndFor
    \ForAll{$z \in [m]$}
        \State Update: \Comment{~\eqref{eq:appr_posterior}, ~\eqref{eq:mu_est}}
        $$
        \log P_{t+1}(z) = \log P_t(z) - \frac{\bigl(R_t-\hat{\mu}_z[A_t]\bigr)^2}{2\sigma^2}.
        $$
    \EndFor
    \State Normalize $\{\log P_{t+1}(z)\}$ \Comment{\textit{log-sum-exp trick}}
    \State $t \gets t+1$
\EndWhile
\end{algorithmic}
\end{algorithm}

\paragraph{Computational Complexity.} The $\lpbTS{}$ algorithm has a computational time complexity of $O(T m k)$, where $T$ is the time horizon, $m$ is the number of latent states, and $k$ is the number of arms. This complexity is primarily from the $O(m k)$ cost per iteration, due to performing isotonic regression for each of the $m$ states on sequences of length $k$, repeated over $T$ iterations. The space complexity is $O(m k)$, dominated by the storing of the ordering matrix $O \in \mathbb{R}^{m \times k}$, with additional $O(k + m)$ space for arm and latent state variables being relatively minor.

\newpage
\section{Two-Stage Latent Preference Order Recovery}
\label{app:latent_recovery}
To recover the Latent Preference Order, we rely on methods for extracting preferences from pairwise comparisons~\citep{bradley1952rank} applied to our setting with latent structure. \citet{vigneau1999analysis} demonstrated that BTMs can be fit for latent structure with an EM approach, and in Figure~\ref{fig:comparison_recovery}, we provide an empirical comparison showing that our two-stage recovery compares favorably with EM Latent Preference Order recovery. 

The following provides an outline of our Preference Order Recovery
\label{sec:ips_recovery}
\paragraph{Extracting pairwise comparisons: }For $N$ instances, logged data with absolute feedback
$
\mathcal{D^\dag}_{T,i} = \{ (A_{t,i}, R_{t,i}) \}_{t=1}^T, \quad i = 1, \dots, N, A_{t,i} \in [k]
$ is used to obtain a dataset incorporating pairwise action comparisons for all instances, captured as
$
\mathcal{D} = \{ (r^{(n)}, Y^{(n)}) \}_{n=1}^N,
$
where $r^{(n)} \in \mathbb{R}^k$ are rewards (possibly incomplete), and $Y^{(n)} \in \{0,1\}^{k \times k}$ has $y_{ij}^{(n)} = 1$ if item $i$ beats $j$ in observation $n$, else 0.
\paragraph{Clustering instances on observed absolute rewards: }Observations are then clustered into $m$ groups $\mathcal{D}_1, \dots, \mathcal{D}_m$ using $\{ r^{(n)} \}$ (via KMeans with zero imputation for incomplete rewards) so that each $n$ belongs to a cluster $z\in[m]$.

\paragraph{Fitting cluster BTMs: } With acces to pairwise comparisons, the preference strengths of the actions in each cluster can be obtained, defined with the \textit{utility}, $\boldsymbol{\beta}^{(z)} = (\beta_0^{(z)}, \dots, \beta_{k-1}^{(z)})$, $\sum_{i=0}^{k-1} \beta_i^{(z)} = 0$ ~\citep{bradley1952rank} where the utility defines the preference strength for each action. We use a Logistic Bradley-Terry model:
$
P(i \succ j \mid z) = \sigma(\beta_i^{(z)} - \beta_j^{(z)}), \quad \sigma(x) = \frac{1}{1 + e^{-x}}.
$

Bradley-Terry models~\citep{bradley1952rank} per cluster are then fit, which give the cluster-data log-likelihood for a cluster $z$ as
$$
\ell^{(z)}(\boldsymbol{\beta}^{(z)}) = \sum_{i \neq j} \Bigl[ y_{ij}^{(z)} \ln \sigma(\beta_i^{(z)} - \beta_j^{(z)}) + \bigl(w_{ij}^{(z)} - y_{ij}^{(z)}\bigr) \ln \Bigl(1 - \sigma(\beta_i^{(z)} - \beta_j^{(z)})\Bigr) \Bigr].
$$
with agregated \textit{cluster} pairwise comparisons
$$
y_{ij}^{(z)} = \sum_{n \in \mathcal{D}_z} y_{ij}^{(n)}, \quad w_{ij}^{(z)} = \sum_{n \in \mathcal{D}_z} I_{ij}^{(n)},
$$
and where $I_{ij}^{(n)} = 1$ if pair $(i,j)$ is observed in $n$.
\paragraph{Post-step Sigmoid:} The real-valued cluster utilities obtained, $\boldsymbol{\beta}^{(z)}$, are then passed through a Sigmoid to restrict them to $[0, 1]$, obtaining $\tilde{\boldsymbol{\beta}}^{(z)} = \sigma(\boldsymbol{\beta}^{(z)}) = \frac{1}{1 + e^{-\boldsymbol{\beta}^{(z)}}}$. This Sigmod step also enables us to model cluster \textit{reward mean vectors $\hat{\mu}^z$}, by scaling $\tilde{\boldsymbol{\beta}}_i^{(z)}$ to an appropriate range, for example $[1, 5]$ for movie ratings with the MovieLens dataset.

A Preference Order $O$ is obtained simply by sorting $\tilde{\boldsymbol{\beta}}^{(z)}$. Algorithm~\ref{alg:ips_logistic} provides a summary for this.

\begin{algorithm}[H]
\caption{Latent Logistic Bradley Terry Model (BTM) for recovering $O$}\label{alg:ips_logistic}
\begin{algorithmic}[1]
\Require Data $\{ (r^{(n)}, Y^{(n)}) \}_{n=1}^N$, number of clusters $m$ 
\State Cluster $\{ r^{(n)}\}_{n=1}^N $ into $\mathcal{D}_1, \dots, \mathcal{D}_m$ (via KMeans) \Comment{Pre-step clustering}
\For{$z = 1$ to $m$}
    \For{$i, j = 0$ to $k-1$, $i \neq j$}
        \State $y_{ij}^{(z)} \gets \sum_{n \in \mathcal{D}_z} y_{ij}^{(n)}$
        \State $w_{ij}^{(z)} \gets \sum_{n \in \mathcal{D}_z} I_{ij}^{(n)}$ 
    \EndFor
    \State $\hat{\boldsymbol{\beta}}^{(z)} \gets \arg\max_{\boldsymbol{\beta}^{(z)}} \ell_{}^{(z)}(\boldsymbol{\beta}^{(z)})$ \textbf{subject to} $\sum_{i=0}^{k-1} \beta_i^{(z)} = 0$
    \State $\tilde{\boldsymbol{\beta}}^{(z)} \gets \sigma(\hat{\boldsymbol{\beta}}^{(z)})$ \Comment{Post-step Sigmoid}
    \State $O_z \gets \argsort(-\tilde{\boldsymbol{\beta}}^{(z)})$
\EndFor
\State \textbf{Output:} $\{ O_z \}_{z=1}^m$, $\{ \tilde{\boldsymbol{\beta}}^{(z)} \}_{z=1}^m$ 
\end{algorithmic}
\end{algorithm}

\begin{figure}
    \centering
    \begin{subfigure}{0.95\textwidth}
    \includegraphics[width=\textwidth]{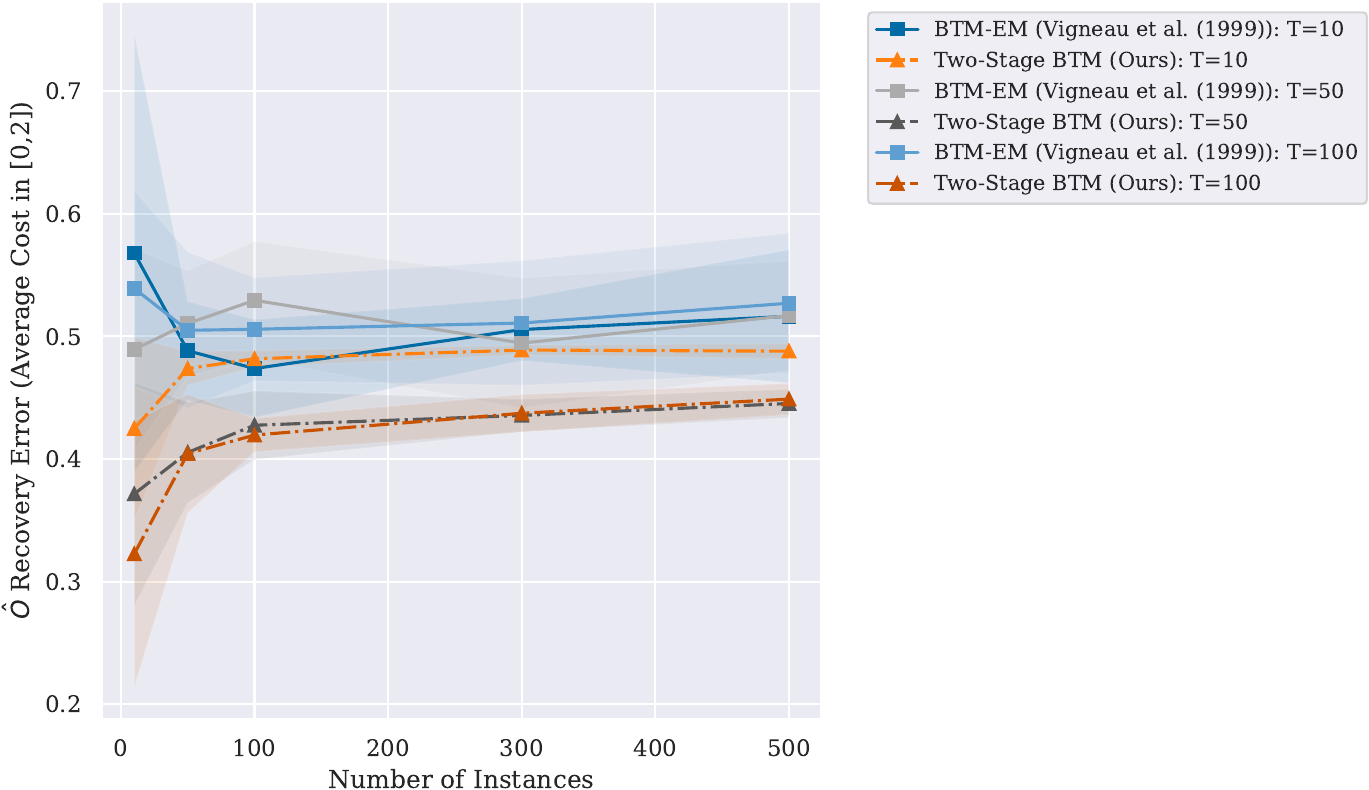}
    \end{subfigure}
    \caption{\textbf{Synthetic Recovery Experiment (O known, generated according to Section~\ref{sec:experiments})}. Illustration of the average matching error between true and recovered orderings, computed as the average of (1 - Kendall's tau correlation) after optimal matching using the Hungarian algorithm. The error decreases as the number of instances increases, indicating improved recovery accuracy with more data. Our two-stage recovery compares favorably to an EM approach.}
    \label{fig:comparison_recovery}
\end{figure}
\clearpage
\newpage
\section{Additional Experiments}
Here we outline additional results from our synthetic ablation study in Figures \ref{fig:combined_all3} and \ref{fig:combined_all2}, and our MovieLens experiments (with 1M Dataset in Figure~\ref{fig:movielens_combined_1m}, and 32M Dataset in Figure~\ref{fig:movielens_combined_32m}).
\label{app:additional_experiments}
\begin{figure}[t]
    \centering
    
    \begin{subfigure}[a]{0.89\textwidth}
        \centering
        \includegraphics[width=\textwidth]{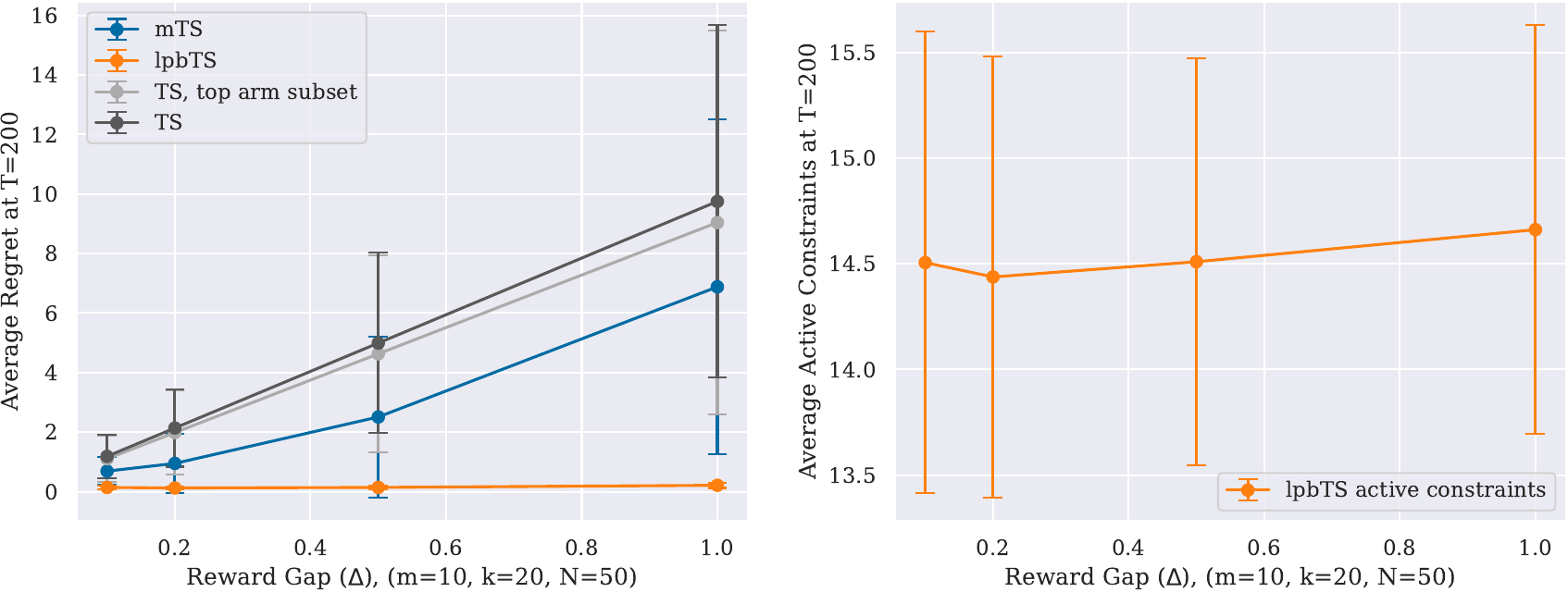}
        \caption{ Varying the reward gap $\Delta~(k=20, m=10, N=50, T=200,  \Delta\in[0.1, 0.2, 0.5, 1.0]$). \textbf{\textit{Left:}}  Observed average regret at $T=200$. \textbf{\textit{Right:}} Observed average active constraints.}
        \label{fig:app_active_D}
    \end{subfigure}
    \vspace{0.5cm} 
    \begin{subfigure}[b]{0.89\textwidth}
        \centering
        \includegraphics[width=\textwidth]{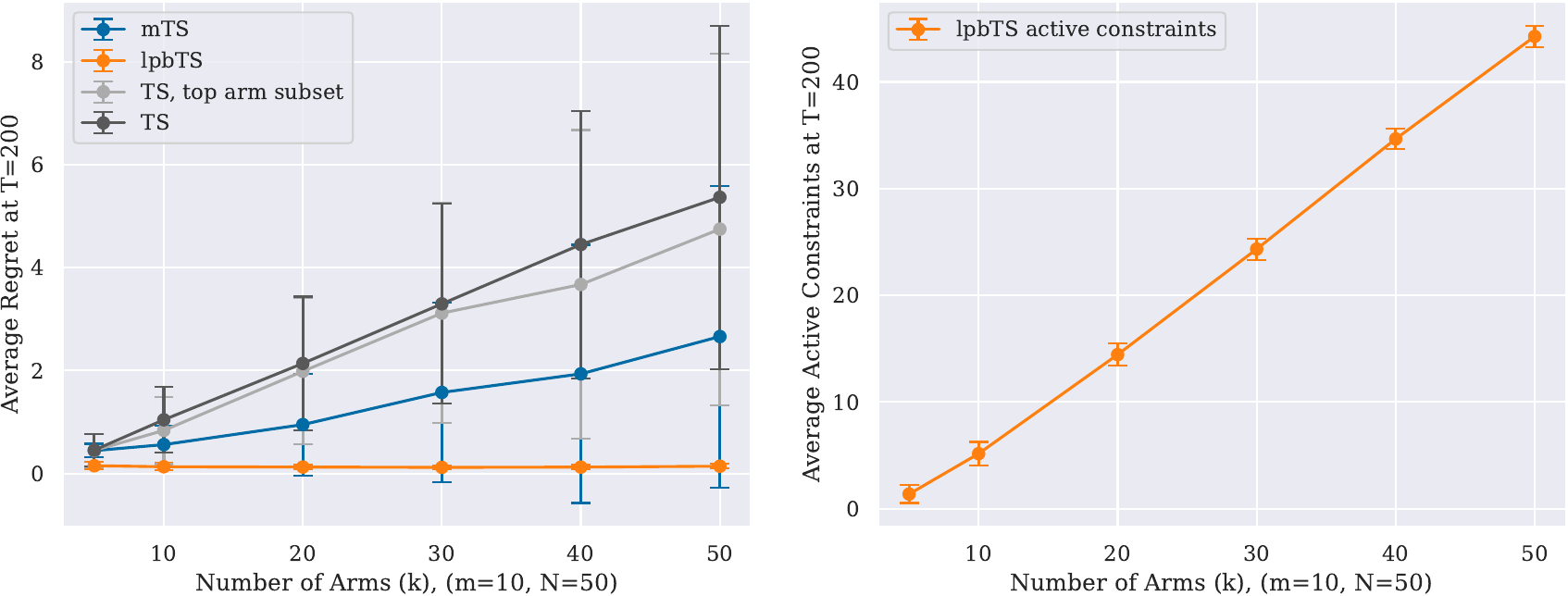}
        \caption{ Varying the number of arms $k~(m=10, N=50, T=200,  k\in[5, 10, 20, 30, 40, 50]$). \textbf{\textit{Left:}}  Observed average regret at $T=200$. \textbf{\textit{Right:}} Observed average active constraints.}
        \label{fig:app_active_k}
    \end{subfigure}
    \hfill
    \begin{subfigure}[c]{0.89\textwidth}
        \centering
        \includegraphics[width=\textwidth]{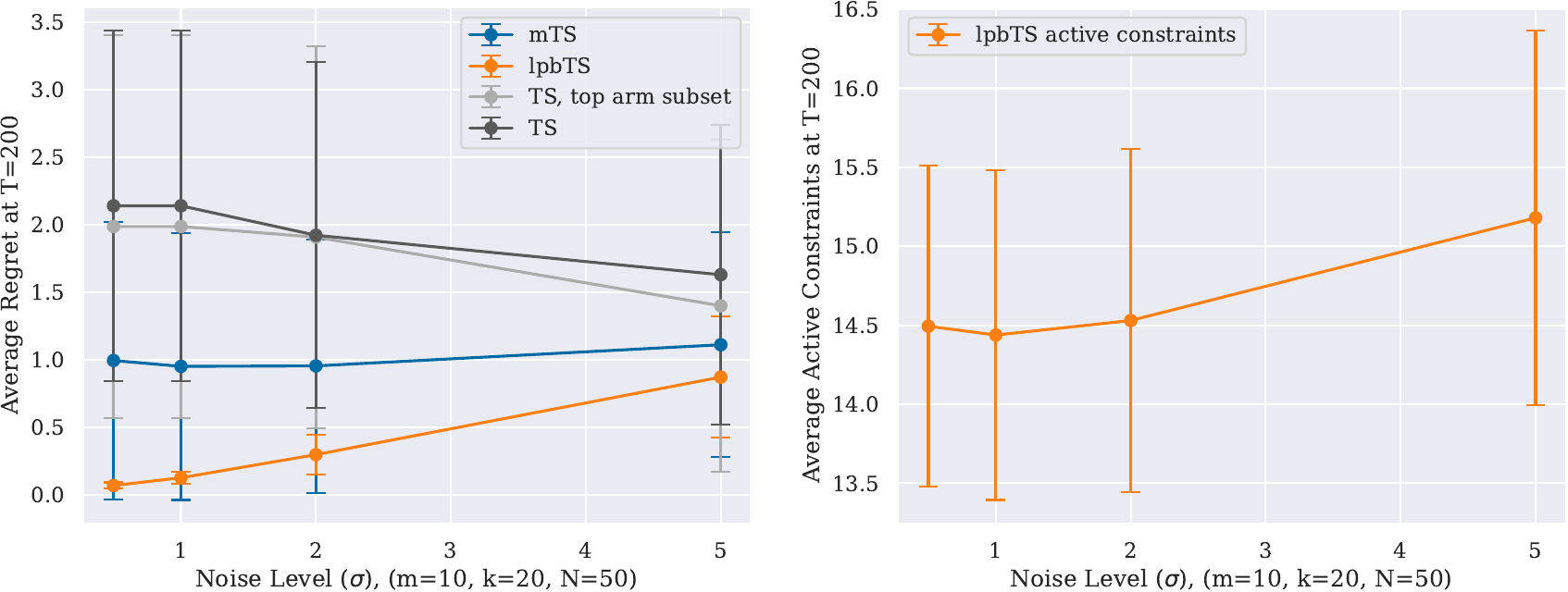}
        \caption{ Varying the reward noise $\sigma~(k=20, m=10, N=50, T=200,  \sigma\in[0.5, 1, 2, 5]$). \textbf{\textit{Left:}}  Observed average regret at $T=200$. \textbf{\textit{Right:}} Observed average active constraints.}
        \label{fig:app_active_sigma}
    \end{subfigure}
    \caption{Synthetic experiment, instance rewards in different scales.  With increasing reward separation $\Delta$, \lpbTS{} maintains a stably low regret by leveraging structural constraints and \mts{} regret only worsens because the the draw interval for the mean reward of the optimal arm grows with $\Delta k$. The active constraints in \lpbTS{} increases as $k$ grows because the number of possible permutations increase like $k!$, so the probability of having large differences between states grows when $k$ grows, and there's need to distinguish more alternative states. When reward noise increases, distinguishing alternative states becomes harder as shown in the increase in active constraints in  \lpbTS{}.}
    \label{fig:combined_all3}
\end{figure}

\begin{figure}[t]
    \centering
    \begin{subfigure}[a]{0.84\textwidth}
        \centering
        \includegraphics[width=\textwidth]{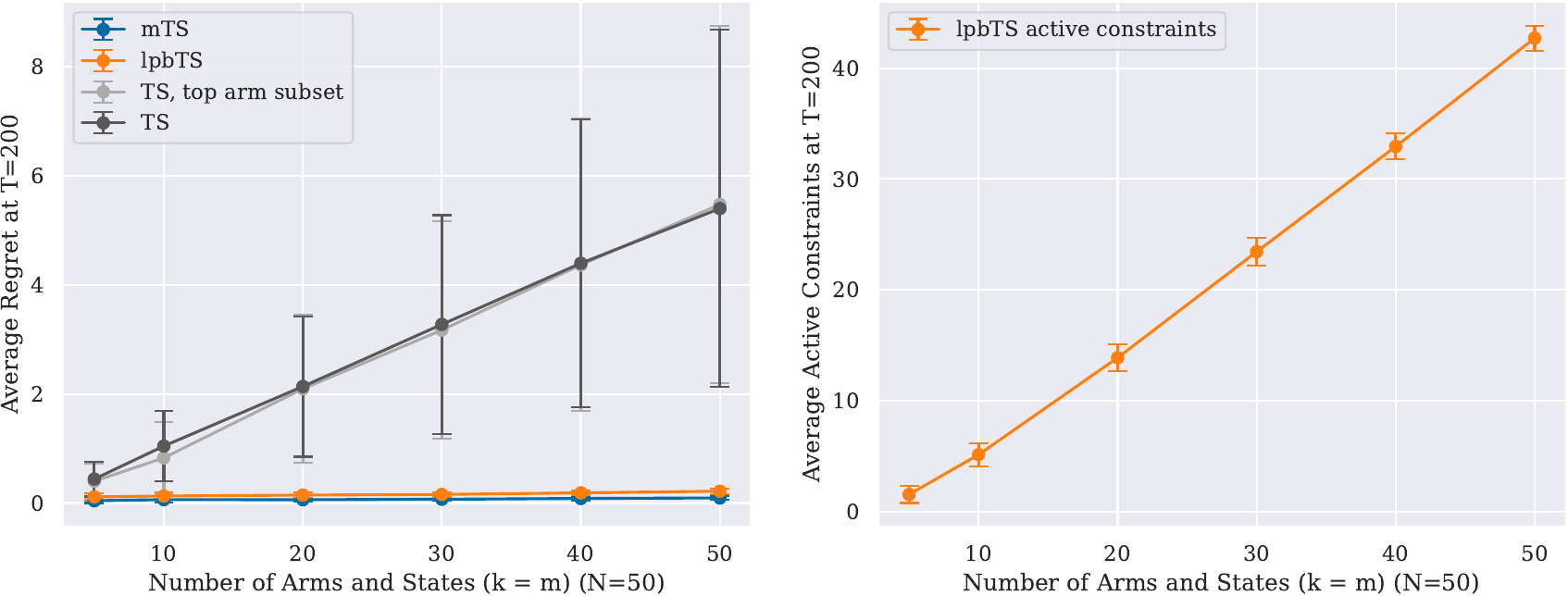}
        \caption{ Varying the number of arms $k$, and $m=k~(N=50, T=200,  k\in[5, 10, 20, 30, 40, 50]$). \textbf{\textit{Left:}}  Observed average regret at $T=200$. \textbf{\textit{Right:}} Observed average active constraints.}
        \label{fig:app_active_k_m}
    \end{subfigure}
    \begin{subfigure}[b]{0.84\textwidth}
        \centering
        \includegraphics[width=\textwidth]{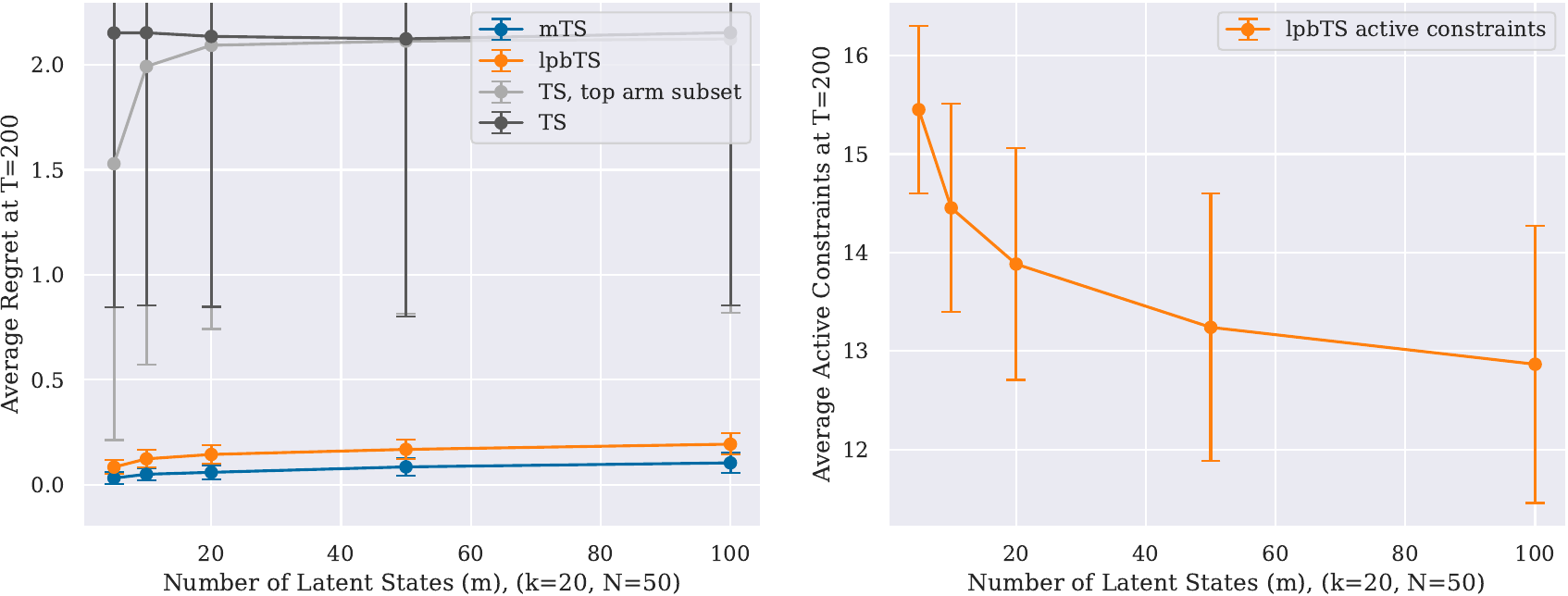}
        \caption{Varying the number of latent states $m$ ($k=20, N=50,  T=200, m\in[5, 10, 20, 50, 100]$). \textbf{\textit{Left:}} Observed average regret at $T=200$ \textbf{\textit{Right:}} Observed average active constraints at $T=200$. }
        \label{fig:app_varying_m_s}
    \end{subfigure}
    \hfill
    \begin{subfigure}[c]{0.84\textwidth}
        \centering
        \includegraphics[width=\textwidth]{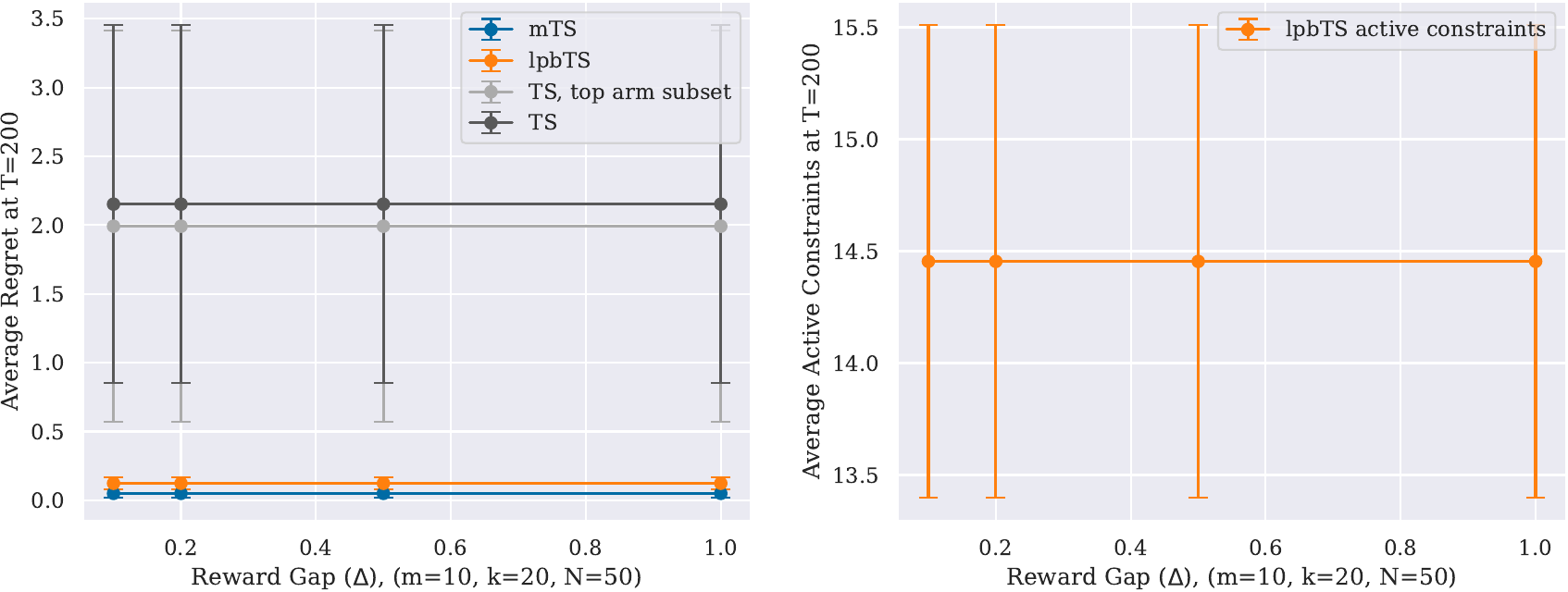}
        \caption{ Varying the reward gap $\Delta~(k=20, m=10, N=50, T=200,  \Delta\in[0.1, 0.2, 0.5, 1.0]$). \textbf{\textit{Left:}}  Observed average regret at $T=200$. \textbf{\textit{Right:}} Observed average active constraints.}
        \label{fig:app_active_D_s}
    \end{subfigure}
    \vspace{0.5cm} 
    \begin{subfigure}[d]{0.84\textwidth}
        \centering
        \includegraphics[width=\textwidth]{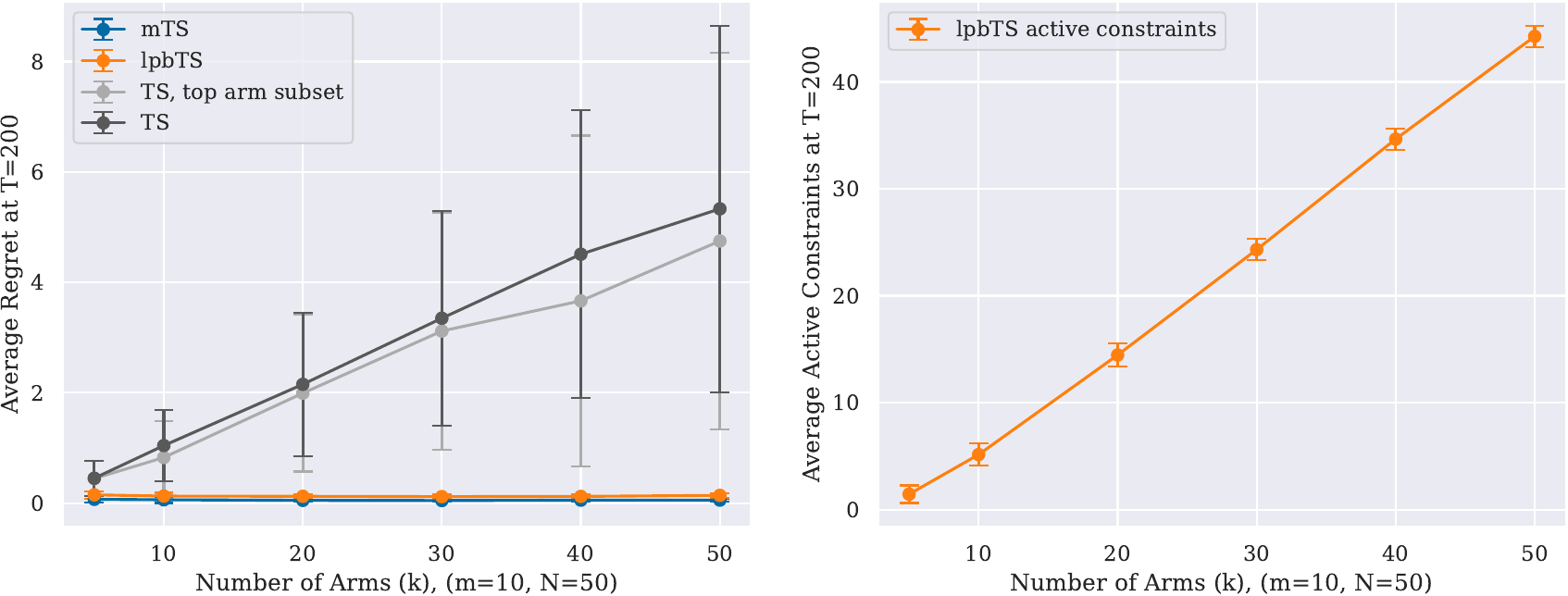}
        \caption{ Varying the number of arms $k~(m=10, N=50, T=200,  k\in[5, 10, 20, 30, 40, 50]$). \textbf{\textit{Left:}}  Observed average regret at $T=200$. \textbf{\textit{Right:}} Observed average active constraints.}
        \label{fig:app_active_k_s}
    \end{subfigure}
    
    \caption{Synthetic experiment, instance rewards in the same scale. Here, \mts{} outperforms in regret due to it's knowledge of true means. Active constraints observed in \lpbTS{} explain the latent structure characteristics described earlier more clearly.}
    \label{fig:combined_all2}
\end{figure}

\begin{figure}[t]
    \centering
    \begin{subfigure}[b]{0.99\textwidth}
        \centering
        \includegraphics[width=\textwidth]{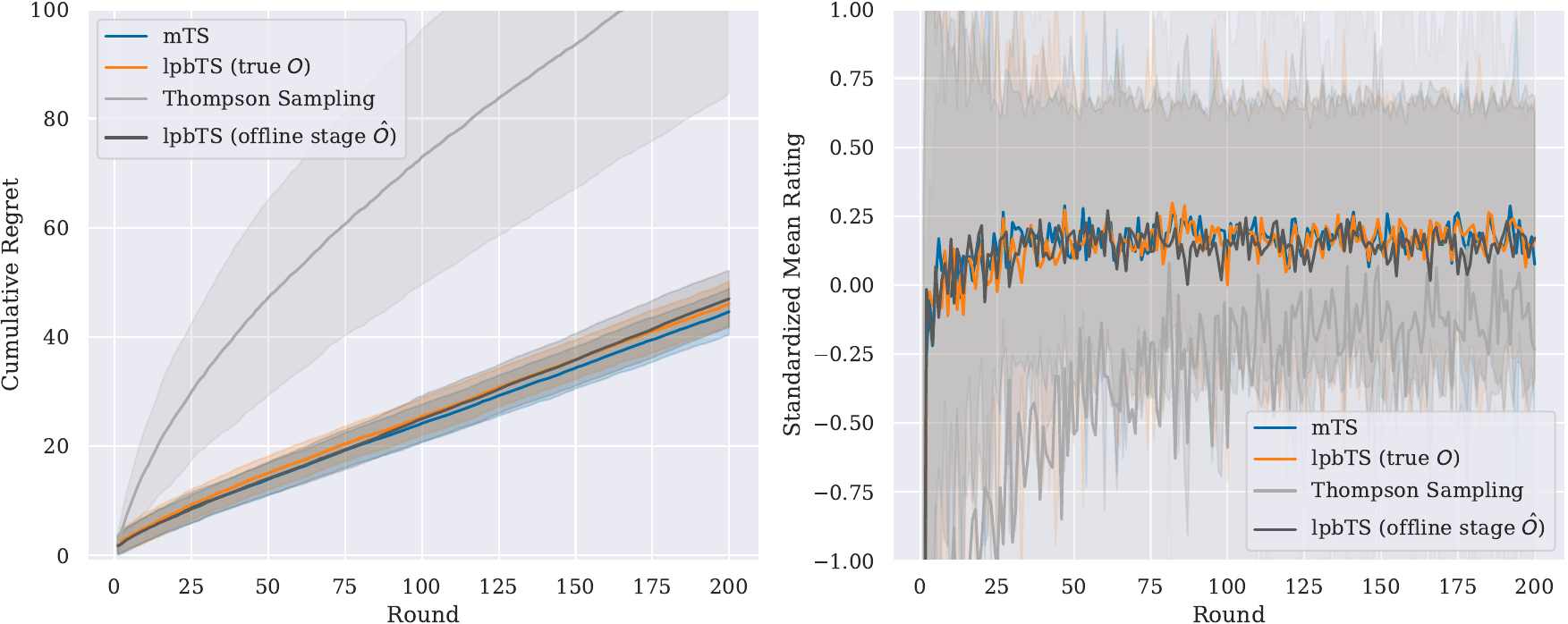}
        \caption{Movie ratings in the same scale for users in a latent state}
        \label{fig:movielens_same_scale_1m}
    \end{subfigure}
    \hfill
    \begin{subfigure}[b]{0.99\textwidth}
        \centering
        \includegraphics[width=\textwidth]{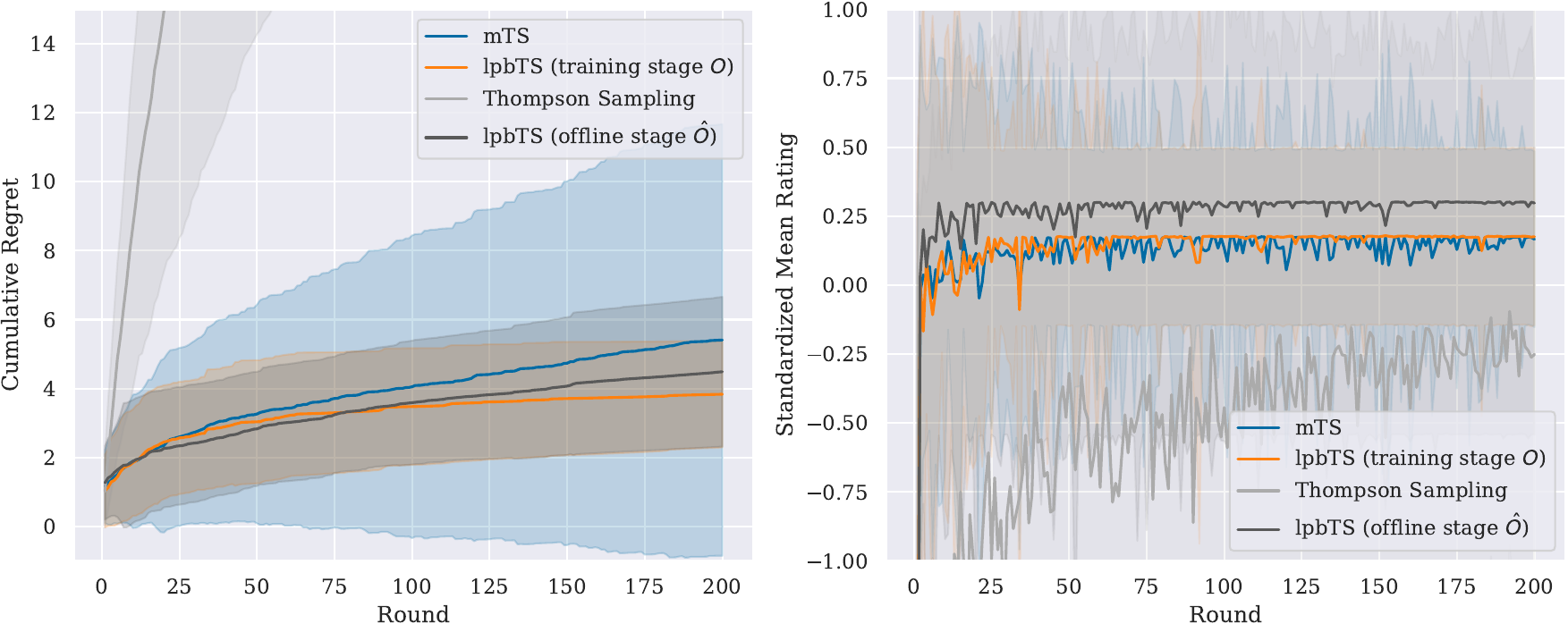}
        \caption{Movie ratings in different scales for users in a latent state}
        \label{fig:movielens_diff_scale_1m}
    \end{subfigure}
     \caption{MovieLens Experiment, 1M Dataset. Results match theory: \lpbTS{~\textbf{(Ours)}} is comparable to mTS in (a), outperforms in (b), and the two-stage recovery of $O$ is empirically validated. Furthermore, we see that when the offline stage $\hat{O}$ is trained with instance reward scales varying, it extracts preferences that are more robust (ground truth $O$ for both was trained with rewards in the same scale).}
    \label{fig:movielens_combined_1m}
\end{figure}

\begin{figure}[t]
    \centering
    \begin{subfigure}[b]{0.99\textwidth}
        \centering
        \includegraphics[width=\textwidth]{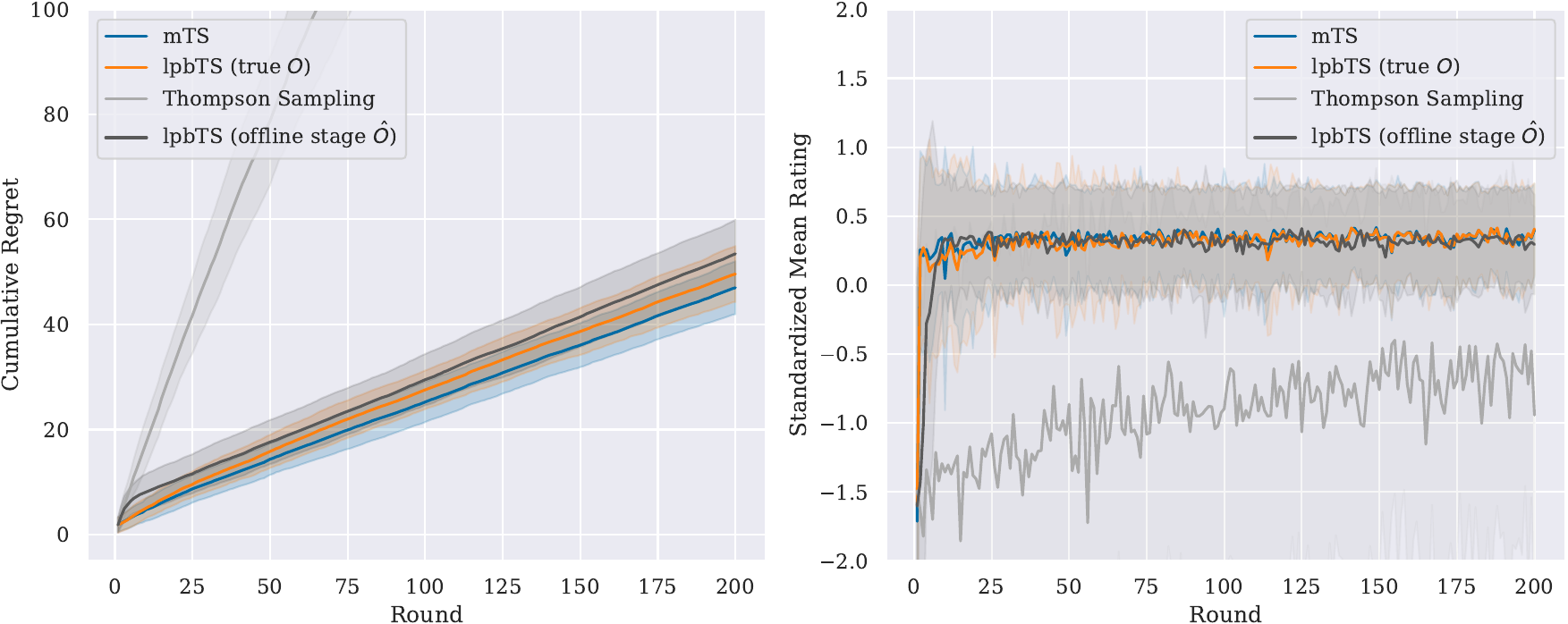}
        \caption{Movie ratings in the same scale for users in a latent state}
        \label{fig:movielens_same_scale_32m}
    \end{subfigure}
    \hfill
    \begin{subfigure}[b]{0.99\textwidth}
        \centering
        \includegraphics[width=\textwidth]{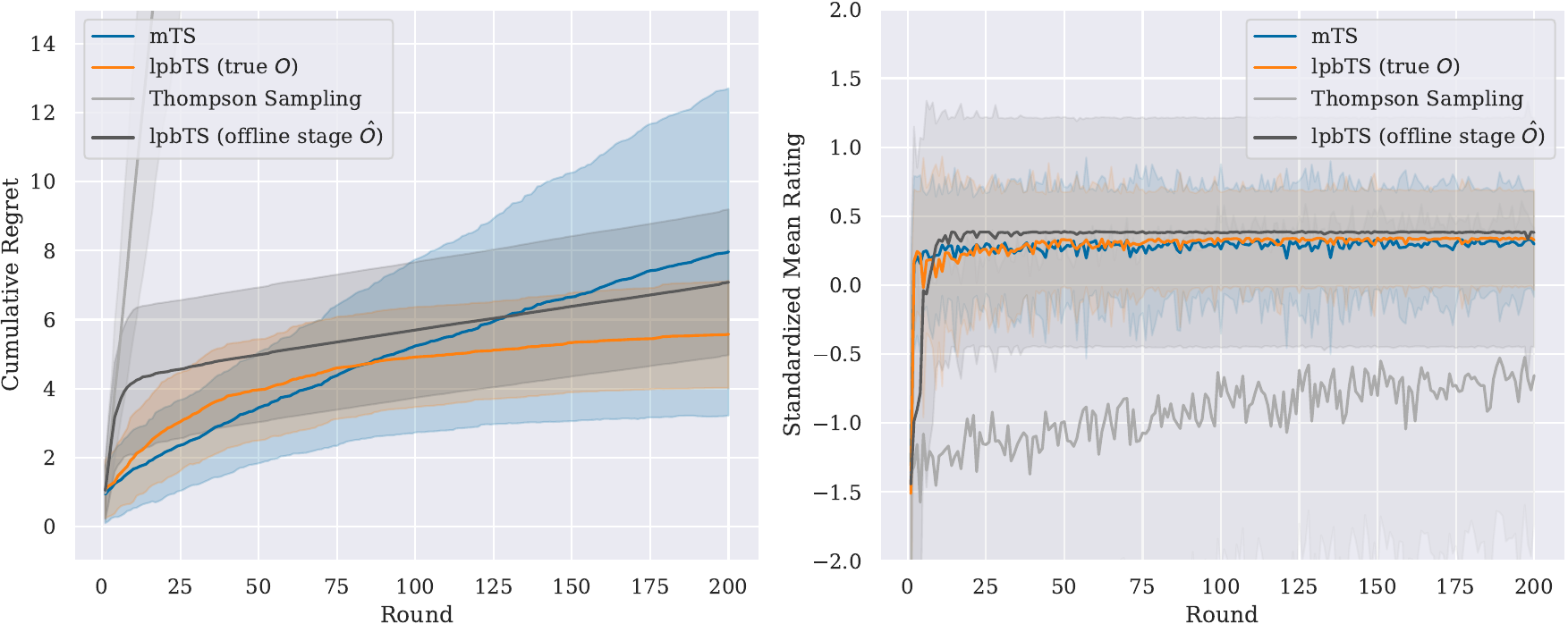}
        \caption{Movie ratings in different scales for users in a latent state}
        \label{fig:movielens_diff_scale_32m}
    \end{subfigure}
     \caption{MovieLens Experiment, 32M Dataset. Results match theory: \lpbTS{~\textbf{(Ours)}} is comparable to mTS in (a), outperforms in (b), and the two-stage recovery of $O$ is empirically validated. Furthermore, we see that when the offline stage $\hat{O}$ is trained with instance reward scales varying, it extracts preferences that are more robust (ground truth $O$ for both was trained with rewards in the same scale).}
    \label{fig:movielens_combined_32m}
\end{figure}
\clearpage
\newpage

\section{Computation Infrastructure and Code}
The synthetic simulations were run on a 2.6 GHz 6-Core Intel Core i7 Macbook laptop, with 16 GB RAM. Ablation studies and MovieLens experiments were done on a compute cluster with 2 nodes, each having an NVIDIA Tesla T4 GPU with 16GB RAM (Each with 4 Intel Xeon Gold 6226R CPU, 2.90GHz and 72 GB DDR4 RAM).

\end{document}